%% file: main_arxiv.tex
\documentclass[letterpaper]{article} 
\usepackage[]{arxiv}  
\usepackage{times}  
\usepackage{helvet}  
\usepackage{courier}  
\usepackage[hyphens]{url}  
\usepackage{graphicx} 
\usepackage{booktabs}
\usepackage{enumitem}
\usepackage{natbib}  
\usepackage{caption} 
\usepackage{algorithm}
\usepackage{algorithmic}
\usepackage{siunitx}
\usepackage{booktabs}
\usepackage{multicol}
\usepackage{multirow}
\usepackage{subcaption}
\usepackage{adjustbox}
\usepackage{tikz, graphicx}
\usepackage{enumitem}
\usepackage{xcolor}           
\newif\ifdraft
\drafttrue                  
\ifdraft
  \newcommand{\sa}[1]{\textcolor{blue}{[#1 \textsc{}]}}
\else
  \newcommand{\sa}[1]{}     
\fi
\ifdraft
  
\else
\fi

\usepackage{etoolbox}        
\usepackage{caption}         
\definecolor{framegray}{gray}{0.1}

\newlength{\pcfigwidth}
\usepackage{amsmath}
\usepackage{amssymb}
\usepackage{amsthm}
\nocopyright

\theoremstyle{definition}
\newtheorem{definition}{Definition}
\newtheorem{lemma}{Lemma}
\newtheorem{theorem}{Theorem}
\newtheorem{proposition}{Proposition}
\newtheorem{corollary}{Corollary}

\title{Tractable Sharpness-aware Regularization of Probabilistic Circuits}

\usepackage{newfloat}
\usepackage{listings}
\usepackage{multirow}   

\DeclareCaptionStyle{ruled}{labelfont=normalfont,labelsep=colon,strut=off} 
\floatstyle{ruled}
\newfloat{listing}{tb}{lst}{}
\floatname{listing}{Listing}
\title{Tractable Sharpness-Aware Learning of Probabilistic Circuits}
\author{
    Hrithik Suresh\equalcontrib\textsuperscript{\rm 1},
    Sahil Sidheekh\equalcontrib\textsuperscript{\rm 2},
    Vishnu Shreeram M.P\textsuperscript{\rm 1},\\
    Sriraam Natarajan\textsuperscript{\rm 2},
    Narayanan C. Krishnan\textsuperscript{\rm 1}
}
\affiliations{
    \textsuperscript{\rm 1} Mehta Family School of Data Science and Artificial Intelligence, Department of Data Science, Indian Institute of Technology Palakkad, Kerala, India\\
     \textsuperscript{\rm 2} Erik Jonsson School of Engineering \& Computer Science, The University of Texas at Dallas, Richardson, TX, USA

}

\usepackage{bibentry}

\begin{document}

\maketitle

\begin{abstract}
Probabilistic Circuits (PCs) are a class of generative models that allow exact and tractable inference for a wide range of queries. While recent developments have enabled the learning of deep and expressive PCs, this increased capacity can often lead to overfitting, especially when data is limited. We analyze PC overfitting from a log-likelihood-landscape perspective and show that it is often caused by convergence to \emph{sharp optima} that generalize poorly. Inspired by sharpness aware minimization in neural networks, we propose a Hessian-based regularizer for training PCs. As a key contribution, we show that the trace of the Hessian of the log-likelihood--a sharpness proxy that is typically intractable in deep neural networks--can be computed efficiently for PCs. Minimizing this Hessian trace induces a gradient-norm-based regularizer that yields simple closed-form parameter updates for EM, and integrates seamlessly with gradient based learning methods. Experiments on synthetic and real-world datasets demonstrate that our method consistently guides PCs toward flatter minima, improves generalization performance. 
\end{abstract}

\section{Introduction}

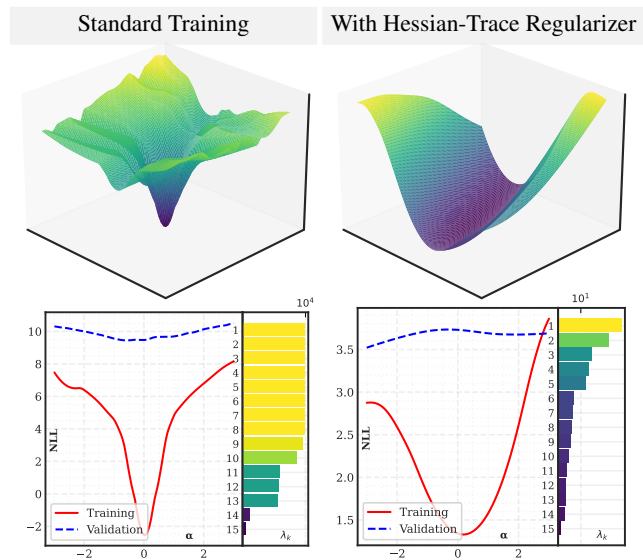
\begin{figure}[t]
    \centering
    \begin{tikzpicture} 
\node[rectangle,draw=white!90,fill=white,opacity=1,minimum width=3.5cm,minimum height=0.5cm] at (0,0) {};
    \node[rectangle,draw=white!90,fill=gray!10,minimum width=4.1cm,minimum height=0.5cm] at (-3.25,0) {\small Standard Training};
    \hspace{-0.7em}\node[rectangle,draw=white!90,fill=gray!10,minimum width=4cm,minimum height=0.5cm] at (1.25,0) {\small With Hessian-Trace Regularizer};
    
    \end{tikzpicture}\\
    \vspace{-0.9em}
    \includegraphics[width=0.47\linewidth, trim=2em 2em 1em 2em,clip]{ 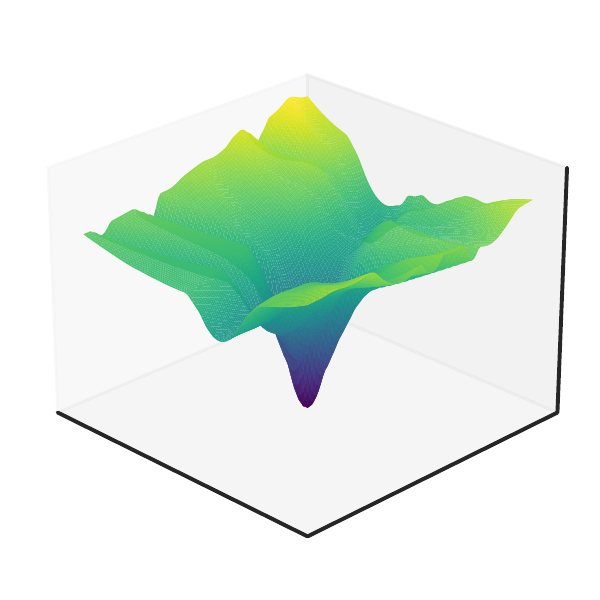}
    \hspace{0.2em}
    \includegraphics[width=0.47\linewidth, trim=2em 2em 1em 2em,clip]{ 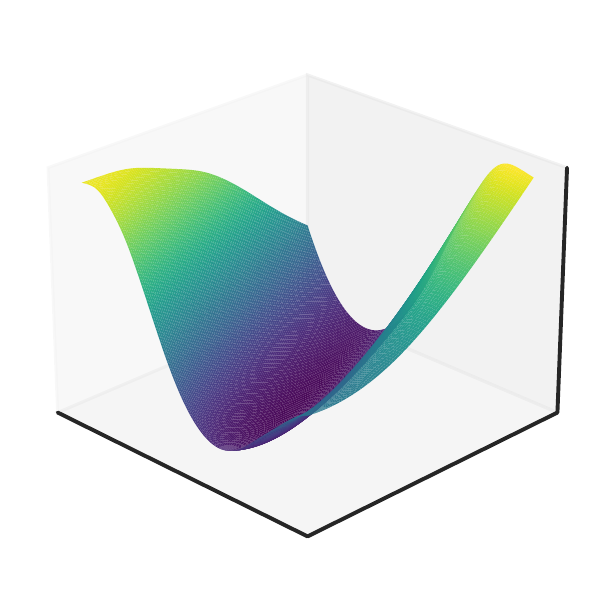}
    \\
    \vspace{-1em}
    \includegraphics[width=0.48\linewidth]{ 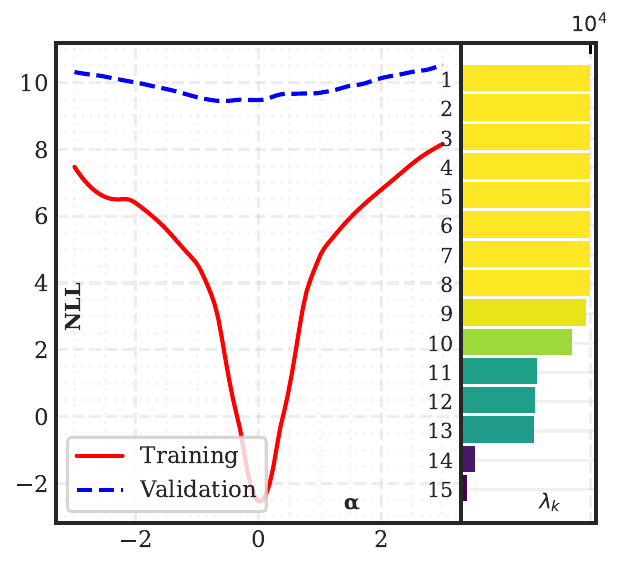}
    \includegraphics[width=0.48\linewidth]{ 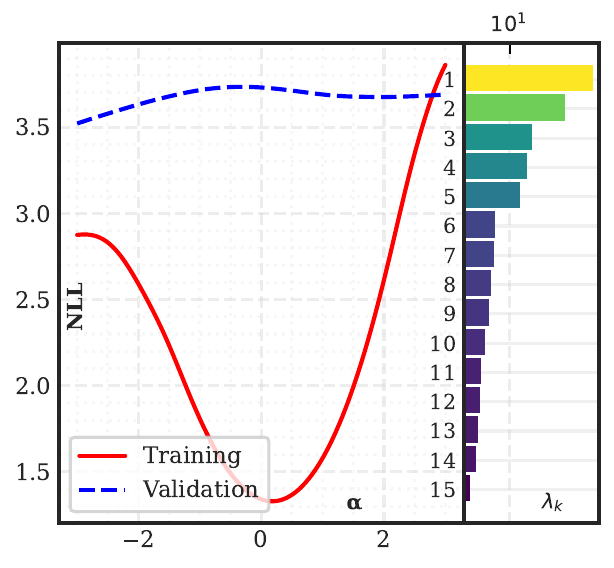}
    \caption{
    Visualization of the 2D (top) and 1D (bottom) loss-landscape (NLL) near the converged parameters of a PC trained with (right) and without (left) our Hessian trace regularizer on a $2$D dataset. Standard training falls into a narrow, sharp basin, while the regularized model settles in flatter minima that generalizes better. The bar plots in the bottom figures depict the top-$15$ eigenvalues of the hessian at the converged point. The lower eigen spectrum on the right quantifies the reduced sharpness achieved by our method.}
    \label{fig:pc-loss-landscape}
\end{figure}

Probabilistic generative models are fundamental to modern machine learning, offering a principled framework for reasoning under uncertainty by modeling data as samples from an unknown underlying distribution. While deep generative models—such as GANs \cite{GoodfellowNeurIPS2014}, VAEs \cite{KingmaICLR2014}, and Normalizing Flows \cite{PapamakariosJMLR2021}—have excelled in generating high fidelity samples, they sacrifice the ability to do exact inference tractably. This limits their usefulness when downstream tasks require calibrated probabilities.
In contrast, Probabilistic Circuits (PCs) \cite{ChoiIMS2020} have emerged as a unifying framework that imposes structural constraints to guarantee 
efficient and exact inference for a rich set of queries \cite{vergari2021compositional}, 
while retaining enough expressivity for real-world applications such as 
constrained generation \cite{zhang2023tractable,zhang2024adaptable}, image inpainting \cite{liu2024image}, lossless compression \cite{liu2022lossless}, multi-modal fusion \cite{sidheekh2025credibilityaware}, and  Neurosymbolic-AI \cite{ahmed2022semantic, ahmed2023semantic, loconte2023turn, karanam2025unified}.


Recent works have therefore pushed towards  building deeper and more expressive PCs \cite{sidheekh2024building}, with millions
of parameters that can be parallelized on GPUs for fast
and efficient training/inference \cite{ZhangICML25scaling,PeharzICML2020}.
However, similar to neural networks, deeper and more expressive PC architectures are increasingly prone to overfitting, especially when trained on limited/noisy data.
Standard parameter-learning methods can often converge to sharp local optima, leading to poor generalization. Such sharp minima, characterized by high curvature, have been extensively studied in deep neural networks, leading to the development of sharpness-aware optimization methods \cite{ForetICLR2021,KwonPMLR2021} that explicitly target flatter minima to enhance generalization.


However, to the best of our knowledge, sharpness-aware learning strategies remain relatively unexplored  for PCs.
We aim to bridge this gap through our work, by studying the geometry of the PC log-likelihood landscape. 
Our key insight is \emph{that the structural properties of a PC permit efficient and exact computation of second-order geometric information}. In particular, we show that the trace of the Hessian of the log-likelihood—which serves as a measure of surface curvature and a proxy for sharpness--
can be computed efficiently in time linear in 
the number of parameters and
the dataset size. This is in stark contrast to deep neural networks, where such exact Hessian computations are intractable in general.

Leveraging this insight,
we introduce a Hessian trace regularizer that  integrates easily into both gradient-based and EM-based training of PCs. Crucially, for EM, we derive the closed-form update rule for the sum node parameters, making our approach scalable and easy to integrate into existing training pipelines.
To provide an intuitive picture of what our approach accomplishes, we visualize the loss landscape around the converged parameters of a PC trained with and without the Hessian trace regularizer in Figure \ref{fig:pc-loss-landscape}, using the filter-normalized projection technique of \citet{li2018visualizing}.
The regularized model settles in a broader and flatter optima compared to standard training, verifying our claim that Hessian trace minimization steers the optimization away from sharp valleys, which in turn delivers stronger generalization.
Overall, in this work, we make the following contributions:
\begin{enumerate}[nosep]
    \item We derive a \textbf{closed form expression for the exact full Hessian} of the log-likelihood for tree-structured PCs and show that it can be computed tractably.
    \item For general (DAG-structured) PCs, we establish that although the full Hessian can be intractable, its trace remains exactly computable in time linear in both the number of parameters and dataset size, \textbf{providing the first practical curvature measure for large-scale PCs.}
    \item We introduce a {\bf novel sharpness-aware regularizer} for learning PCs, derived from  this Hessian trace.
    \item We show that while directly minimizing the Hessian trace via EM leads to a cubic update equation, we can reformulate this objective into an equivalent gradient norm minimization problem, {\bf resulting in a quadratic equation with closed-form parameter updates}.
    \item We conduct exhaustive experiments on multiple synthetic and real-world datasets to {\bf show that our regularizer enforces convergence to flatter optima and helps reduce overfitting}, especially in limited data settings.
\end{enumerate}

\section{Background and Preliminaries}

\begin{definition}[]
A \textbf{Probabilistic Circuit} $p$  is a parameterized directed acyclic graph (DAG) with a unique root node \( n_r \) that compactly encodes a joint probability distribution over a set of random variables \( \mathbf{X} = \{X_1, \ldots, X_d\} \). 
It is composed of three types of nodes: \textit{Input nodes (leaf)} representing simple univariate distributions over a single variable, \textit{Sum nodes (internal)} that compute a weighted sum of its children's output, and \textit{Product nodes (internal)} that compute a product of its children's output. Formally, each node \( n \) in the DAG computes a distribution \( p_n \), defined recursively as follows:
\[
p_n(x) = 
\begin{cases}
    f_n(x), & \text{if } n \text{ is an input node} \\
    \prod_{c\in in(n)} p_{c}(x), & \text{if } n \text{ is a product node} \\
    \sum_{c\in in(n)} \theta_{nc} \cdot p_c(x), & \text{if } n \text{ is a sum node }
\end{cases}
\]
where $f_n$ is a univariate input distribution (e.g., Bernoulli, Gaussian, etc), $in(n)$ denotes the children of $n$ and $\theta_{nc}$ is the weight parameter on the edge $(n, c)$, 
such that $\forall c \in in(n) \  \theta_{nc} \in (0, 1]$ and $\sum_{c \in in(n)} \theta_{nc} = 1$. 
\end{definition}

\begin{definition}[Scope]
The scope function $\phi$ associates to each node in the PC a subset of $\mathbf{X}$, i.e., $\phi(n) \subseteq \mathbf{X}$, over which it defines a distribution. For each non-terminal node $n$, $\phi(n)=\cup_{c\in in(n)}\phi(c)$. The scope of the root $n_r$ is $\mathbf{X}$    
\end{definition}
The sum and product nodes represent convex mixtures and factorized distributions over the scopes of their children, respectively. 
A PC is evaluated bottom up and the joint distribution is computed as the output of its root node, i.e. \( p(x) = p_{n_r}(x) \). The size of $p$, denoted $|p|$, is the number of edges in its DAG. We make the common assumption that $p$ contains alternating sum and product node layers. This formalism subsumes several classes of tractable models such as arithmetic circuits \cite{darwiche2003differential-ac}, sum-product networks \cite{PoonICCV2011}, PSDDs \cite{KisaKR2014} and cutset networks \cite{rahman2014cutset}.

To achieve tractability for exact marginal (MAR), conditional (CON) and maximum-a-posteriori (MAP) inference, a PC has to satisfy certain \textbf{structural properties} ~\cite{ChoiIMS2020}, some of which are:
\begin{definition}[Smoothness]
A PC is smooth if the children of every sum node $n$ have the same scope: $\forall c_1, c_2 \in in(n), \phi(c_1) = \phi(c_2)$.
\end{definition}
\begin{definition}[Decomposability]
A PC is decomposable if the children of every product node $n$ have disjoint scopes: $\forall c_1, c_2 \in in(n), \phi(c_1) \cap \phi(c_2) = \emptyset $. 
\end{definition}
\begin{definition}[Determinism]
Define the support $supp(n)$ of a PC node $n$ as the set of complete variable assignments $x\in val(\textbf{X})$ for which $p_n(x) > 0$. A PC is deterministic if the children of every sum node $n$ have disjoint support: $\forall c_1, c_2 \in in(n), c_1 \neq c_2, supp(c_1) \cap supp(c_2) = \emptyset$. 
\end{definition}
Smoothness ensures that each sum node represents a valid mixture, while decomposability allows integrals (or sums) to factorize recursively for tractable MAR and CON inference. Adding 
determinism further makes MAP inference tractable. However enforcing structural properties often reduces the model's expressivity. Thus, recent works have aimed to increase their expressivity by efficiently scaling them using tensorized implementations \cite{PeharzUAI2019,PeharzICML2020,LiuICML2024, loconte2025what}, borrowing inductive biases from deep generative models \cite{SidheekhPMLR2023,liu2023scaling, correia2023continuous, gala2024probabilistic} and relaxing structural assumptions \cite{loconte2024subtractive,loconte2025sum,wang2025relationship}.

Regardless of these advances, learning the PC parameters is predominantly achieved using one of two standard paradigms: stochastic gradient based optimization or expectation-maximization (EM).
As differentiable computational graphs, PCs allow efficient computation of gradients of the likelihood (or log-likelihood) w.r.t their parameters via backpropagation. 
Thus stochastic gradient descent (SGD) and its variants can be used directly to learn their parameters.
Alternatively, one can view each sum node as introducing a latent categorical variable
indexing its child edges and apply EM to optimize the incomplete data log-likelihood. In this framework, the E-step computes posterior distributions over edges, while the M-step updates the weights in closed form, given the posterior.
Notably, both SGD and EM admit a unified formulation that can be understood in terms of the notion of \emph{circuit flow} \cite{LiuNeurIPS2021}.
\begin{definition}[Circuit Flow]
The flow associated with the nodes of a PC is defined in the following recursive manner
\[
F_n(x)= 
\begin{cases}
     1,& \text{if n is the root node }\\
    \sum\limits_{c\in pa(n)}  F_c (x) ,& \text{if n is a input/sum node}\\
    \sum\limits_{c\in pa(n)} \theta_{n c}\frac{p_n(x)}{p_c(x)} F_c(x),& \text{if n is a product node}
\end{cases}
\]
\label{def:circuitflow}
\end{definition}
The flow associated with an edge $(n, c)$ is defined as \( F_{nc}(x) = \theta_{nc} \frac{p_c(x)}{p_n(x)} F_n(x)\). In the EM interpretation, $F_{n c}(x)$ corresponds to the expected count of how often the edge $(nc)$ is used (E-step). The M-step then maximizes $\sum_{c\in \mathrm{in}(n)}\;F_{nc}(x)\,\log\theta_{nc}
\quad\text{s.t.}\quad
\sum_{c\in \mathrm{in}(n)}\theta_{nc}=1$, which for a mini-batch $D_i$ yields the closed form update \(\theta^{\text{mini}}_{nc}
=\frac{\sum_{x\in D^i}F_{nc}(x)}
       {\sum_{j\in \mathrm{in}(n)}\sum_{x\in D^i}F_{nj}(x)}\).
       To smooth out mini-batch noise, a running average is often used:
\(
\theta^{\text{new}}_{nc} = (1 - \alpha)\, \theta^{\text{old}}_{nc} + \alpha\, \theta^{\text{mini}}_{nc}, \quad \alpha \in [0, 1].
\)
Under gradient-based learning, flows have the interpretation that $F_{n c}(x) = \theta_{n c} \frac{\partial \log P_{n_r}(x)}{\partial \theta_{n c}}$.
Computing the flows 
requires only a single forward-backward pass of the PC, and has been proven effective for learning PCs with hundreds of millions of parameters at scale \cite{LiuICML2024}.

However, such large and expressive PCs can overfit when data is limited, prompting \textbf{regularization strategies} that adapt ideas from both deep learning and graphical models. 
Probabilistic dropout \cite{PeharzUAI2019} randomly masks inputs and sum‑node children during training to simulate missing data and mixture uncertainty, reducing co-adaptation among sub-circuits. Pruning and re-growing \cite{DangNeurIPS2022} has been suggested to remove redundant sub‑networks, yielding sparser models that generalize better. Classical Laplacian smoothening on the sum-node weights have also been applied to PCs \cite{LiuNeurIPS2021}, although naive Laplace priors can bias the mixtures when child supports are imbalanced. \citet{shih2021hyperspns} observed that deep PCs can have tens of millions of parameters and proposed a hypernetwork that generates sum-node weights from low dimensional embeddings, thereby reducing the free parameters while preserving expressivity. Recent work has also tried to exploit the tractability of PCs to propose customized regularizers. \citet{ventola2023probabilistic} adapted the idea of Monte Carlo dropout \cite{gal2016dropout} to PCs by deriving tractable dropout inference that propagates variances through the circuit in a single pass to obtain better calibration and out-of-distribution detection. \citet{LiuNeurIPS2021} proposed \emph{data softening}, which replaces each training example with a locally blurred distribution and an entropy regularizer on the circuit's global output distribution.
However, this requires solving a non-linear equation via Newton’s method at each sum node, leading to added implementation and computational complexity during training.

In parallel, the deep-learning community has demonstrated that convergence to sharp minima--regions of high curvature--often correlates with poor generalization, leading to the development of \textbf{sharpness-aware optimization strategies}.
While the idea that flatter minima can help unlock higher generalization capability in deep neural networks (DNNs) has been around for a long time \cite{hochreiter1997flat}, recent works have shown how to achieve this in practice using sharpness aware minimization (SAM) \cite{ForetICLR2021}, which solves a local minimax problem to find parameter updates robust to worst-case perturbations.
Extensions such as Adaptive SAM (ASAM) \cite{KwonPMLR2021} have also been proposed to adaptively scale the perturbations using curvature information. 
We {\bf posit that the relevance of sharpness-aware techniques extends beyond DNNs to PCs}, where training objectives like log-likelihood maximization are often non-convex and prone to overfitting in overparameterized regimes. In such settings, sharpness-aware learning can not only offer a principled means to promote solutions that generalize better,
but the structure inherent in PCs can enable computing such curvature information \emph{exactly} and \emph{efficiently}.

\section{Sharpness-Aware Learning for PCs}
The Hessian matrix--the second order partial derivatives of a loss function--has long been used as a natural way to quantify flatness, as its eigenvalues capture the curvature along different directions. \citet{BottcherJSM2024} 
used dominant eigenvectors to visualize the loss landscape of DNNs and 
distinguish sharp minima from flat ones, while
\citet{ChaudhariICLR2017} used this idea to guide DNNs towards wider optima. More recently, \citet{KaurPMLR2023}  proposed using the largest eigenvalue as a flatness metric, while \citet{SankarAAAI2021} developed a layer-wise Hessian trace-based regularizer. However, as computing the full Hessian is intractable for DNNs, most methods rely on implicit curvature estimates, such as rank-$1$ approximations \cite{MartensICML2012} and the Hutchinson trace estimator \cite{HutchinsonCSSC1990}. In contrast, as we detail below, the structured DAG of a PC permits exact and efficient curvature computation, enabling a true-sharpness aware regularizer without resorting to costly approximations.

\subsection{Full Hessian 
Computation for Tree-Structured PCs}
A tree-structured PC (in short TS-PC) is one where every non-root node in its DAG has exactly one parent, and hence there is a unique path from the root ($n_r$) to any node ($n$). Our first result is that for a TS-PC, the Hessian of the log-likelihood with respect to the parameters can be computed tractably.
Figure \ref{fig:ts-pc-structure} illustrates a typical \(n_r\)--\(n\) path in a TS-PC. 
Let  $S^l_i(x)$ and $P^l_i(x)$ denote the outputs of the $i^{th}$ sum and product nodes at level $l$, respectively.
Note that $x$ would contain only the subset of variables defined in the scope of the node. We ignore it in the notation for clarity.
To see how this structure yields closed-form Hessian entries, consider the task of expressing the gradient of the root likelihood $P_{n_r}(x)$ w.r.t a mixing weight $\theta_{nc}$. We can unfold the computation along a unique path 
\(
n_r  \to  \cdots  \to P^l_i \to  S^{l}_j  \to  \cdots n\to c
\). At each product node $P^l_i$ on this path, the contribution of its other children, i.e. those not on the path enters as a multiplicative factor. We will refer to the product of such \emph{sibling} outputs as the \emph{product complement}, as defined below:

\begin{definition}[Product Complement]
The product complement of a product node $P^l_i$ with respect to one of its children $S^{l-1}_j$ is defined as:
\[
    \overline{P}^l_{ij}(x) = \prod_{\substack{k \in in(P^l_i)\\k \neq j}} S^{l-1}_k(x)
\]
\end{definition}
All such product complements in a PC can be computed in a single forward pass. Recall that a single forward-backward pass also computes the \emph{circuit flow} $F_n(x)$ (Def. \ref{def:circuitflow}) for each node $n$ and $F_{nc}(x)$ for each edge $(n,c)$, which can be used to compute the gradients. The unique-path property of a TS-PC  collapses the summations in the circuit flow recursion into a simpler chain, resulting in compact expressions for the flow and gradient, as presented below.

\begin{figure}
    \centering
    \includegraphics[width=\linewidth]{ 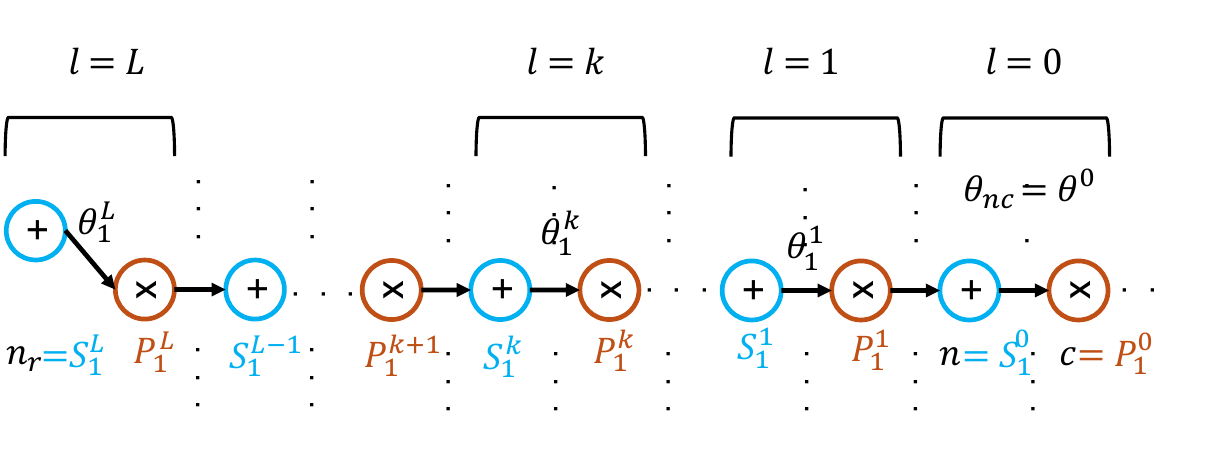}
    \caption{A typical path structure in a tree-structured PC.}
    \label{fig:ts-pc-structure}
\end{figure}
\begin{lemma}
Consider the unique path from the root $n_r$ to a sum node $n$ in a TS-PC: \(n_r \to P^L_1\to S^{L-1}_1\to \cdots\to P^1_1\to S^0_1 = n\),  as shown in Figure \ref{fig:ts-pc-structure}. Then the flow at node $n$ is given by:
\[
F_n(x) = \theta^1_1 \frac{P^1_1(x)}{P_{n_{r}}(x)}\prod_{j=2}^L\theta_1^j\bar{P}^j_{11}(x)
\]
\end{lemma}
\begin{proof}
The proof involves unrolling the flow at a sum node using the tree structure of a TS-PC and the notion of product complements
of the product nodes along the path from the node to the root, and is provided in the appendix.
\end{proof}
\begin{corollary}
The flow associated with the edge $(n,c)$ parametrized by  $\theta^n_c$ ( = $\theta^0_1$ in Figure \ref{fig:ts-pc-structure}) is given by:
\[
    F_{nc}(x) = \theta_{nc} \dfrac{P_c(x)}{P_{n_r}(x)}\prod_{l=1}^L\theta^l_1 \bar{P}^{l}_{11}(x)   
\]
and correspondingly, the gradient of the likelihood at the root node with respect to $\theta_{nc}$ is
\[
    \dfrac{\partial P_{n_r}(x)}{\partial \theta_{nc}} =   P_c(x) \prod_{l=1}^L\theta^l_1 \bar{P}^{l}_{11}(x)
\]
\label{cor:derivative-in-terms-of-product-complement}
\end{corollary}
\begin{proof}
The proof is a straight forward application of the flow defined on an edge.
\end{proof}
For notational simplicity we have labeled every sum and product‐node along our canonical path by index “$1$”.  In a general TS‑PC path, the nodes at layer \(l\) would carry their own index \(i_l\), but the same formulas hold by replacing each “$1$” with the appropriate \(i_l\).
The key insight from the above corollary is that the partial derivative of the likelihood with respect to
$\theta_{nc}$ factorizes into exactly the product complements and weights along the unique path from $n$ to the root, and \emph{does not depend on any sum-node outputs} on that path.

Let $\theta_{n c}$ and $\theta_{n' c'}$ be the weights associated with two distinct edges in a TS-PC. Due to the tree-structure, $n$ and $n'$ either lie on the same path from the root node or share exactly one deepest common ancestor, let us call it $q$. Thus, the pair $(\theta_{n c}, \theta_{n' c'})$ belongs to one of the following three cases: 1) \textbf{Sum pair}: if $q$ is a sum node. 2) \textbf{Product pair}: if $q$ is a product node. 3) \textbf{Path pair}: if the 
$\theta_{n c}$ and $\theta_{n' c'}$ lie on the same path from the root.
Our main result shows that the second derivative of root likelihood w.r.t the PC parameters can be expressed in closed form for each of the above cases.
\begin{theorem}
    \label{lemma:Likelihood-Hessian}
    The mixed second derivative  \( \frac{\partial^2 P_{nr}(x)}{\partial \theta_{n'c'} \partial \theta_{nc}}\)  of the output of a tree-structured PC with respect to any two parameters \(\theta_{nc}\) and \(\theta_{n'c'}\)  equals
\begin{align*}
    \begin{cases}
     0, &\text{if } 
     \text{ it is a sum pair}\\
   \dfrac{\dfrac{\partial P_{n_r}(x)}{\partial \theta_{nc}}\dfrac{\partial P_{n_r}(x)}{\partial \theta_{n'c'}}}{\theta^k_1 P^k_{1}(x)\left(\prod_{l=k+1}^L\theta^l_1 \bar{P}^{l}_{11}(x)\right)} ,& \text{if } 
   \text{it is a product pair}\\ 
    \dfrac{1}{\theta_{n'c'}} \cdot \dfrac{\partial P_{n_r}(x)}{\partial \theta_{nc}},& \text{if }
    \text{ it is a path pair with}\\ & \text{ $\theta_{n'c'}$ closer to the root}
\end{cases}
\end{align*}
\end{theorem}
\noindent where $P_{1}^k$ denotes the deepest common product node, and we follow our canonical notation for the path to the root.
\begin{proof}
    Deferred to the appendix.
\end{proof}
Theorem \ref{lemma:Likelihood-Hessian} can be further extended to obtain equally simple expressions for the Hessian of the \emph{log-likelihood} which is typically used as the objective for training PCs, as follows:
\begin{proposition}
If ($\theta_{n c}$, $\theta_{n'c'}$) is a sum pair, then
\begin{equation}
\dfrac{\partial^2 \log P_{n_r}(x)}{\partial \theta_{n'c'} \partial \theta_{nc}} = -\dfrac{F_{nc}(x)}{\theta_{nc}} \dfrac{F_{n'c'}(x)}{\theta_{n'c'}}
\end{equation}
\end{proposition}
\begin{proof}
    Deferred to the appendix.
\end{proof}
Thus, the mixed second derivative of the log-likelihood for a sum pair
factorizes into the product of the corresponding first-order derivatives. 
\textit{A special case is when the sum pair corresponds to the same parameter ($\theta_{nc} = \theta_{n'c'}$), in which case the double derivative equals the square of the gradient of the log-likelihood with respect to $\theta_{nc}$.}

\begin{proposition}
    If ($\theta_{n c}$, $\theta_{n'c'}$) is a product pair with $P^k_1$ being their deepest common ancestor, then
    \small
    \begin{align*}
     \dfrac{\partial^2 log P_{n_r} (x)}{\partial \theta_{n'c'}\partial \theta_{nc}} = &\dfrac{P_{n_r}(x)\dfrac{F_{nc}(x)}{\theta_{nc}} \dfrac{F_{n'c'}(x)}{\theta_{n'c'}}}{\theta^k_1 P^k_{1}(x)\left(\prod_{l=k+1}^L\theta^l_1 \bar{P}^{l}_{11}(x)\right)} \\ &- \dfrac{F_{nc}(x)}{\theta_{nc}} \dfrac{F_{n'c'}(x)}{\theta_{n'c'}}
    \end{align*}
\end{proposition}
\begin{proof}
    Deferred to the appendix.
\end{proof}
\begin{proposition}
    If ($\theta_{n c}$, $\theta_{n'c'}$) is a path pair, with $\theta_{n'c'}$ closer to the root node, then
    \begin{align*}
\frac{\partial^2 \log P_{n_r}(x)}{\partial \theta_{n'c'} \partial \theta_{nc}} &= \frac{\partial^2 \log P_{n_r}(x)}{\partial \theta_{nc} \partial \theta_{n'c'}}\\ 
&= \frac{1}{\theta_{n'c'}} \cdot \frac{F_{nc}(x)}{\theta_{nc}}- \frac{F_{nc}(x)}{\theta_{nc}}\frac{F_{n'c'}(x)}{\theta_{n'c'}}
    \end{align*}
\end{proposition}
\begin{proof}
    Deferred to the appendix.
\end{proof}

Since each entry of the full Hessian depends only on the circuit flow, the mixing parameters and the product complements--all of which can be computed efficiently--{\bf the Hessian as a whole can likewise be evaluated tractably}.

\subsection{Hessian for General (DAG-Structured) PCs}
To understand whether the tractability of Hessian computation extends to  general DAG-structured PCs, we examine its diagonal and off-diagonal entries separately. Our analysis suggests that while the former can be computed efficiently, the latter can suffer from a combinatorial explosion. We present the results and defer proofs to the supplementary.

\begin{proposition}
The diagonal entry of the Hessian of the log-likelihood of a general PC w.r.t a parameter $\theta_{nc}$ is given by:
\[
\dfrac{\partial^2 \log P_{n_r}(x)}{\partial^2 \theta_{nc}} = -\left(\dfrac{F_{nc}(x)}{\theta_{nc}}\right)  ^2 
\]
\label{Prop:Hessian_trace}
\end{proposition}
\begin{proof}
   Deferred to the appendix.
\end{proof}
From Proposition \ref{Prop:Hessian_trace}, we observe that computing the trace only requires access to the edge flows and the mixing parameters. From Definition \ref{def:circuitflow}, all edge flows can be computed with a single forward and backward pass through the circuit, making the flow computation linear in time with respect to the number of edges. Consequently, the overall trace computation is also linear in time.
However, for general PCs, we conjecture that computing the off-diagonal entries of the Hessian is intractable due to the exponential number of dependency paths between parameters. Concretely, consider a PC where each internal node (except the root and its immediate children) has up to $w$ parents. Suppose that sum nodes $n$ and $n'$ share $w$ deepest common ancestors at the same depth, and let $d^*$ denote the number of layers between these deepest common ancestors and the lower of the two nodes, i.e.,
$d^* = \min(\text{depth}(n), \text{depth}(n')) - \text{depth}(\text{deepest common ancestors}).
$ Then, the number of paths connecting $n$ and $n'$ can grow as $O(w^{d^*})$. This exponential growth in the number of shared paths indicates that, in such densely connected PCs, computing off-diagonal entries of the Hessian becomes computationally intractable. 

\subsection{A Tractable Sharpness Regularizer for PCs}
Although the full Hessian can be intractable for arbitrary PCs, its trace remains efficiently computable, and serves as a scalar measure of the overall curvature of the log-likelihood surface--large values indicating sharper optima, while lower values correspond to flatter optima \cite{KeskarICLR2017,ForetICLR2021,KwonPMLR2021}. Thus,
reducing the Hessian trace during training can serve as an effective regularization strategy.
While the full Hessian can be computed efficiently and potentially incorporated as a regularizer for tree-structured PCs, we focus instead on its trace, as it is both simpler to compute and applicable to the general class of PCs.
Crucially, for any PC (not just tree-structured),
the absolute trace (ignoring \textit{absolute} henceforth) 
simplifies to the sum of squared partial derivatives:
\begin{align*}
\mathrm{Tr}\left( \nabla^2 \log P_{n_r}(x) \right) &= \sum_{n,c} \left( \dfrac{\partial \log P_{n_r}(x)}{\partial \theta_{nc}} \right)^2 \\&= \sum_{n,c} \left( \dfrac{F_{nc}(x)}{\theta_{nc}} \right)^2          
\end{align*}
This enables sharpness-aware regularization using only first-order derivatives that can be computed using the edge flows, while still promoting flatter solutions during training.
A simple way to incorporate this into gradient-based learning is to add the Hessian trace as a regularizer 
\(R(\theta, x)=\sum_{n,c}(F_{nc}(x)/\theta_{nc})^2\) 
to the negative log-likelihood, yielding the  objective \(\min_\theta \;-\sum_{x\in D}\log P_{n_r}(x)
\;+\;\lambda\sum_{x\in D}R(\theta,x).\) Since $R(\theta,x)$ depends only on the local flows and weights, its gradients can be computed with a forward-backward pass, making integration with optimizers like SGD or Adam straightforward. However, EM is often preferred over gradient descent to learn PCs as it achieves faster convergence \cite{desana2017}.
Thus, we next discuss how to integrate this curvature penalty into EM-based learning.

\input{ tables/einsum-final}
\subsubsection{Sharpness-Aware EM.}
To endow EM with sharpness awareness, we propose to add the 
Hessian-trace 
regularizer into the M-step 
by constraining the sum squared gradients at each sum node. More formally,
the M-step optimization is now carried out under two constraints: (1) the parameters at each sum node must lie on the probability simplex, i.e., $\sum_{c \in in(n)} \theta_{nc} = 1$, and (2) the trace of the Hessian is upper bounded, i.e., $\sum_{c \in in(n)} \left( F_{nc}(x)/\theta_{nc} \right)^2 \leq m$ for some $m$. The resulting parameter update takes the following form:

\begin{proposition}
The EM update for a parameter \(\theta_{nc}\) at a sum node \(n\), under a Hessian trace-based sharpness regularizer, is the solution to the cubic equation:
\begin{equation}
    \lambda\, \theta_{nc}^3 - F_{nc}(x)\, \theta_{nc}^2 - 2\,\mu F_{nc}(x)^2 = 0,
\end{equation}
where \(F_{nc}(x)\) is the expected edge flow, and \(\lambda\), \(\mu\) are Lagrange multipliers for the normalization and trace constraints, respectively.
\end{proposition}
\begin{proof}
   Deferred to the appendix.
\end{proof}

Incorporating the Hessian trace into the EM update thus yields, at each sum node, a cubic equation in the parameters that must be solved exactly. However, solving
such cubic equations 
can be computationally expensive and may yield multiple real roots, negative values, or even complex solutions, making the update process unstable or infeasible. To overcome this, we exploit a key property of PC gradients:
the partial derivative of the log-likelihood with respect to a sum node parameter $\theta_{nc}$ takes the form $\frac{F_{nc}(x)}{\theta_{nc}}$, which is always non-negative as both 
$F_{nc}(x)$ and $\theta_{nc}$
are positive. Consequently, the squared gradient is a monotonic function of the gradient itself, thus minimizing the gradient suffices to reduce its square. This allows us to directly penalize the gradient as a surrogate for reducing the trace, resulting in a simpler \emph{quadratic} update with a closed‐form solution.
\begin{theorem}
The EM parameter update at sum node \(n\), under the gradient regularized objective is given by:
\[
\theta_{nc} = \frac{F_{nc}(x) + \sqrt{F_{nc}(x)^2 + 4 \lambda \mu F_{nc}(x)}}{2 \lambda},
\]
where \(F_{nc}(x)\) denotes the flow along the edge from sum node \(n\) to its child \(c\), and \(\lambda, \mu \ge 0\) are the Lagrange multipliers corresponding to the normalization and regularization constraints, respectively.
\end{theorem}
\begin{proof}
    Deferred to the appendix.
\end{proof}

\section{Experiments and Results}
We organize our empirical evaluation to answer the following four research questions.
\begin{itemize}[wide, labelindent=0pt]
    \item[({\bf Q1})] Do large, expressive PCs overfit on limited data, and does sharpness, defined via Hessian-trace capture this?
    \item[({\bf Q2})] Is our derived sum-squared-gradient expression for the Hessian-trace  correct and efficient to compute?
    \item[({\bf Q3})] Does the proposed sharpness aware-learning framework reduce overfitting and improve generalization?
    \item[({\bf Q4})] What effect does the regularization strength $\mu$ have?
\end{itemize}

\begin{figure}[t]
    \centering
    \begin{subfigure}{0.51\linewidth}
    \includegraphics[width=\linewidth]{ 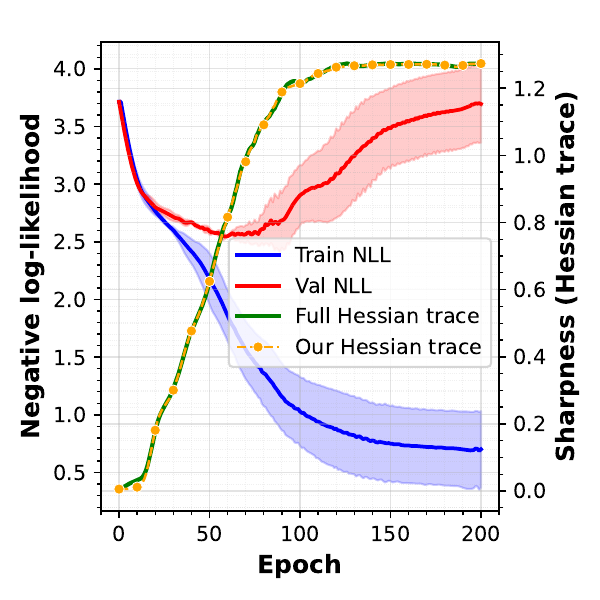}    
    \end{subfigure}
    \begin{subfigure}{0.48\linewidth}
    \includegraphics[width=\linewidth]{ 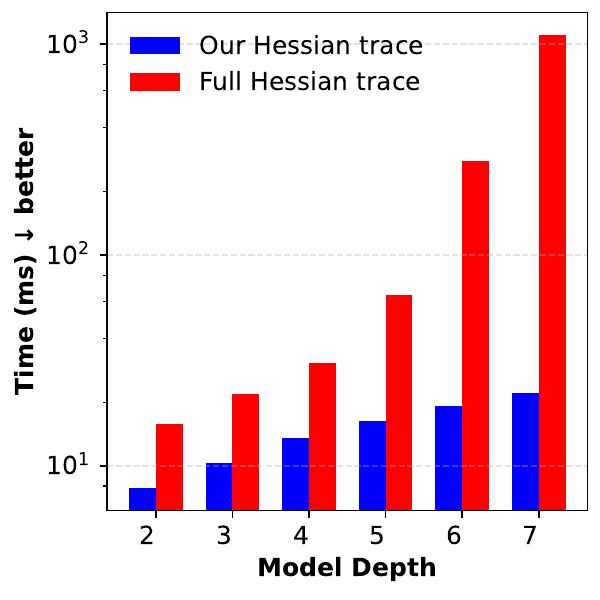}    
    \end{subfigure}
    \caption{[Left] Training and validation negative log-likelihood of EinsumNet on a $2D$ spiral dataset across epochs, with Hessian-trace computed by our sum-squared-gradients formula and torch autograd. [Right]  Time taken for computing the Hessian trace as the network depth grows.}
    \label{fig:Hessian_trace_valid}
\end{figure}

\input{ tables/pyjuice-final}

\paragraph{Setup.} To answer these, we conducted experiments on $8$ synthetic $2$D/$3$D manifold datasets \cite{sidheekh2022vq,SidheekhPMLR2023}
as well as the 20 standard binary density estimation benchmark \cite{van2012markov, bekker2015tractable}.
To show that our approach applies broadly across different PC model classes and implementations,
we integrated it into two widely used PC frameworks—Einsum Networks \cite{PeharzICML2020} and PyJuice \cite{LiuICML2024}, evaluating it on different structural settings.
For synthetic datasets, we use a fixed RAT-SPN \cite{PeharzUAI2019} architecture with $10$ input-distributions, sum-nodes and num-repetitions. For the binary datasets, we adopt the Hidden Chow-Liu Tree (HCLT) structure from PyJuice, with a latent size of $100$ to increase model capacity. To simulate limited data settings where overfitting can happen, we train each model on random subsets of each dataset at fractions \{$1\%$, $5\%$, $10\%$, $50\%$ and $100\%$ \}. 
To show applicability across learning methods, we integrated our regularizer into two settings: (1) gradient-based learning for Einsum Networks (using Adam) and (2) EM-based learning for PyJuice HCLTs (using our quadratic closed-form updates).   
We defer further experimental results and details to the supplementary material.

\paragraph{(Q1) Overfitting and Sharpness:} To show that deep PCs indeed overfit when data is scarce, we plot the train and validation negative log-likelihoods (NLL) for an EinsumNet trained on a $5\%$ data for $2D$ spiral distribution in Figure \ref{fig:Hessian_trace_valid}[left]. We see that the train-NLL continues to decrease, while the val-NLL rises, indicating overfitting and a widening generalization gap. Crucially, the value of sharpness, computed via Hessian-trace also grows in tandem with this gap, peaking when overfitting occurs. This confirms that  sharp minima that correlate with overfitting do exist in PCs, and that the Hessian-trace is capable of detecting them.

\paragraph{(Q2) Correctness \& Efficiency of Hessian-Trace Computation.} We empirically validated the correctness of our Hessian trace derivation using Einsum Networks, which support full Hessian evaluation via PyTorch’s automatic differentiation. Figure~\ref{fig:Hessian_trace_valid}[left] shows an exact match between the Hessian trace computed using our proposed sum-of-squared-gradients (SSG) formula and the one obtained directly via autograd on the $2$D spiral dataset, where full Hessian computation was feasible without exceeding memory limits. This confirms the correctness of our derivation. Figure \ref{fig:Hessian_trace_valid}[right] also compares the computation time of the Hessian trace using autograd and our closed-form SSG formula. We see that as the model depth increases, autograd's runtime suffers from an exponential blow up, while our SSG formula scales only linearly, and is thus a more practical and accurate way to analyze the curvature, even for deep PCs.

\paragraph{(Q3) Gains from Sharpness-Aware Learning.}
To study the effect our sharpness regularizer has in learning better PCs, 
we measured the performance improvement achieved by our regularized model as compared to its vanilla counterpart using three metrics-
\begin{enumerate}[label={(\arabic*)}]
    \item The relative reduction in Test-NLL [$\Delta NLL$(\%)]
    \item The reduction in the degree of overfitting [$\Delta DoF$(\%)], where $DoF=\frac{NLL_{test}-NLL_{train}}{|NLL_{train}|}$ 
    \item The relative reduction in sharpness [$\Delta Sharp$(\%)], as measured by our Hessian-trace formula, at convergence. 
\end{enumerate}

Table \ref{tab:synthetic_sgd_combined}  reports the mean results over five runs for an EinsumNet trained on the synthetic $2$D/$3$D manifold datasets at varying training-set fractions. We observe that in the lowest data setting, on average, our regularizer cuts overfitting by up to $65$ \%, flattens the loss surface by $89$\%, and boosts test log-likelihood by upto $49$\%. Although the absolute gains diminish with more data, they remain positive across all fractions on average, demonstrating that trace minimization consistently guides the learning toward better-generalizing optima. Table \ref{tab:binary_em_combined} presents analogous results for PyJuice HCLTs on real-world binary datasets using our regularized EM. Again, in the lowest-data regime we observe a $7$ \% improvement in test NLL, an $8$\% reduction in overfitting, and a $23$\% decrease in sharpness, on average. As dataset size grows, these gains plateau—and at the highest data fractions we record a marginal ($<0.5 \%$) drop in test NLL. This is expected when overfitting is negligible as the regularizer may push the parameters away from an otherwise sufficient optimum. 
But even in these settings, our method continues to reduce overfitting and sharpness, confirming its effectiveness at steering PCs toward flatter, more robust solutions.  
\begin{figure}[t]
    \centering
    \includegraphics[width=\linewidth,trim=0em 0em 0em 2.5em,clip]{ 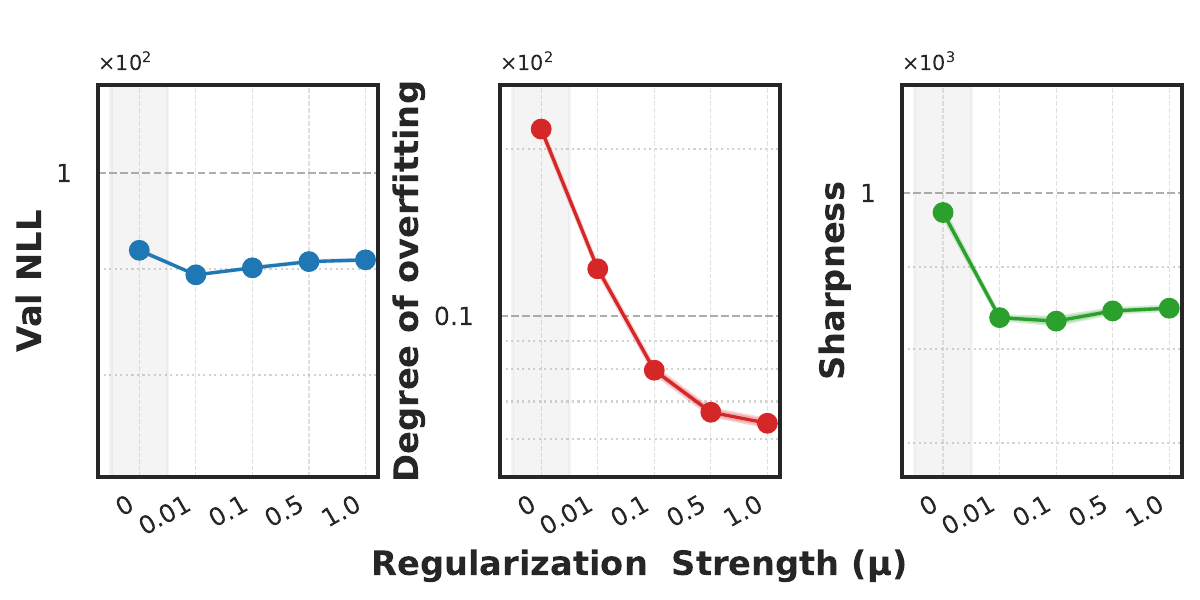}
    \caption{Ablation showing the effect of $\mu$ on the validation negative log-likelihood, degree of overfitting and sharpness.}
    \label{fig:ablation-dna-em-pyjuice-0.05}
\end{figure}
\paragraph{(Q4) Ablation on $\mu$.} To study the sensitivity of our framework to the regularization strength, we trained a PyJuice HCLT using the $5\%$ split of DNA binary dataset, for varying values of $\mu$ and tracked changes in Validation-NLL, degree of overfitting and sharpness (in Figure \ref{fig:ablation-dna-em-pyjuice-0.05}). We observe that even a small $\mu \in (0,0.1]$ is sufficient to capture most of the gains-reducing overfitting and curvature. Larger $\mu$ values yield only marginal gains at the cost of slight underfitting. Thus, our framework is robust to the choice of $\mu$ in a broad mid-range and we select it based on validation performance. 

\section{Conclusion}
In this work, we introduced a new direction to study 
the training of PCs through the lens of the log-likelihood surface geometry. We derived a closed-form expression for the exact full Hessian of the log-likelihood in tree-structured PCs and demonstrated its tractability. For general DAG-structured PCs, we showed that while the full Hessian can be intractable, its trace remains exactly
and efficiently computable--
offering the first scalable curvature measure for training large PCs. 
Building on this, we designed a novel regularizer whose equivalent 
gradient-norm formulation yields closed-form quadratic updates, enabling efficient optimization. Our experiments confirmed that our approach steers training toward flatter minima and reduces overfitting, especially in low-data regimes. Overall, our work opens up a promising new direction for studying PCs. We forsee future work investigating the log-likelihood landscape to identify the presence of asymmetric valleys analogous to those observed in DNNs, developing a theoretical framework for understanding convergence in over-parameterized PCs, and designing alternative optimization strategies that leverage tractable second-order geometric information.

\section*{Acknowledgements}
SN and SS gratefully acknowledge the generous support by the AFOSR award FA9550-23-1-0239, the ARO award W911NF2010224 and the DARPA Assured Neuro Symbolic Learning and Reasoning (ANSR) award HR001122S0039. CK and HS gratefully acknowledge Anagha Sabu for the discussions related to the work and CK, HS, and VS thank IIT Palakkad for the access to Madhava Cluster.

\bibliography{standardized-references}

\end{document}


\appendix
\onecolumn

\noindent \hrulefill
\section*{Supplementary Material: \\ Tractable Sharpness-Aware Learning of Probabilistic Circuits}

\noindent \hrulefill
\vspace{1em}

\section{Proofs for Theoretical Results}
In this section we provide complete, self-contained proofs for all of the key lemmas, corollaries and theorems stated in the main paper.  We begin by showing how the unique path structure of a Tree-Structured Probabilistic Circuit (TS-PC) allows us to ``unroll’’ circuit flows into products of edge weights and \emph{product complements}.  We then derive closed-form expressions for the first and second derivatives in TS-PCs, before turning to the tractable Hessian trace computation for general (non-tree) PCs.

\paragraph{Notation and Preliminaries.} Recall that for any PC node $n$ and input $x$ we write $F_n(x)$ for its \emph{circuit flow} (Definition~6 in the main text), and for any sum-node parameter $\theta_{nc}$ the corresponding edge flow is
\[
F_{nc}(x) \;=\;\theta_{nc}\,\dfrac{P_c(x)}{P_{n}(x)}\,F_n(x).
\]
We will also use the notion of a \emph{product complement} and \emph{product double complement} as defined below:

\begin{definition}[Product Complement]
Let $P^l_i$ be a product node at layer $l$ with children $\{S^{l-1}_j\}_{j\in in(P^l_i)}$.  Its product complement with respect to  \ child $S^{l-1}_j$ is defined as
\[
\overline P^{\,l}_{ij}(x)\;=\;\prod_{k\in in(P^l_i),\,k\neq j} S^{l-1}_k(x)\,.
\]
\end{definition}

\begin{definition}[Product Double Complement]
The product double complement of a product node $P^l_i$ with respect to  two of its children $S^{l-1}_j$ and $S^{l-1}_k$ is defined as:
\begin{equation}
    \overline{\overline{P}}^l_{ijk}(x) = \prod_{\substack{m \in in(P^l_i)\\m \neq {j,k}}} S^{l-1}_m(x)
\end{equation}
In other words, the product double complement of a product node with respect to two of its children is the product of the outputs of all the other children except the two under consideration. 
\end{definition}
\subsection{Full Hessian Computation for Tree Structured PCs}

\begin{figure}
    \centering
    \begin{subfigure}[b]{\linewidth}
        \centering
        \includegraphics[width=0.25\linewidth]{ 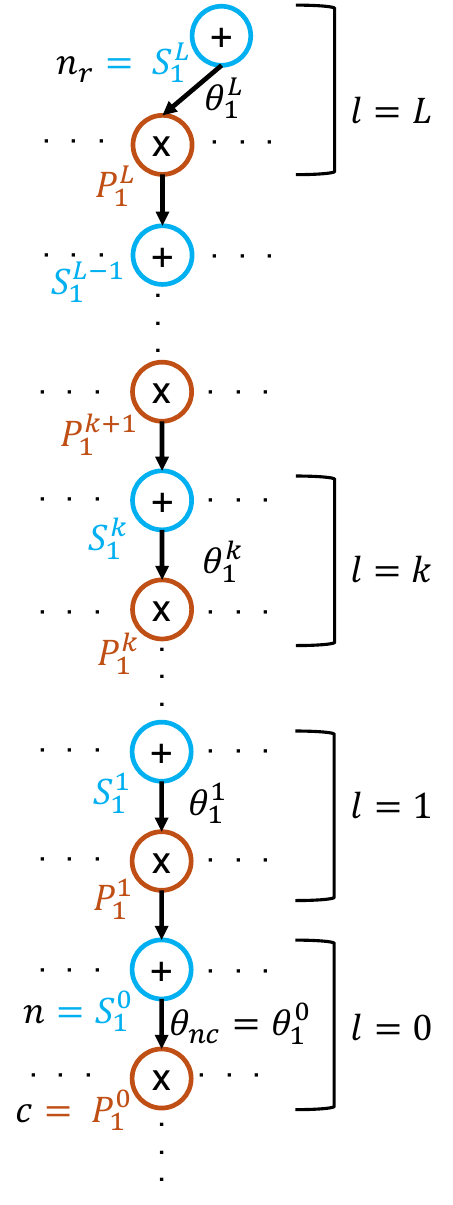}
        \caption{A typical (unique) path structure from the root node $n_r (=S_1^L)$ to a sum node $n (=S_1^0)$ in a tree structured PC. We group the alternating sum and product node layers into a single joint level $l$ for ease of analysis.}
        \label{fig:ts-pc-path-structure}
    \end{subfigure}\\
    \begin{subfigure}[b]{0.3\linewidth}
        \centering
        \includegraphics[width=\linewidth, trim=33em 0em 33em 0em, clip]{ 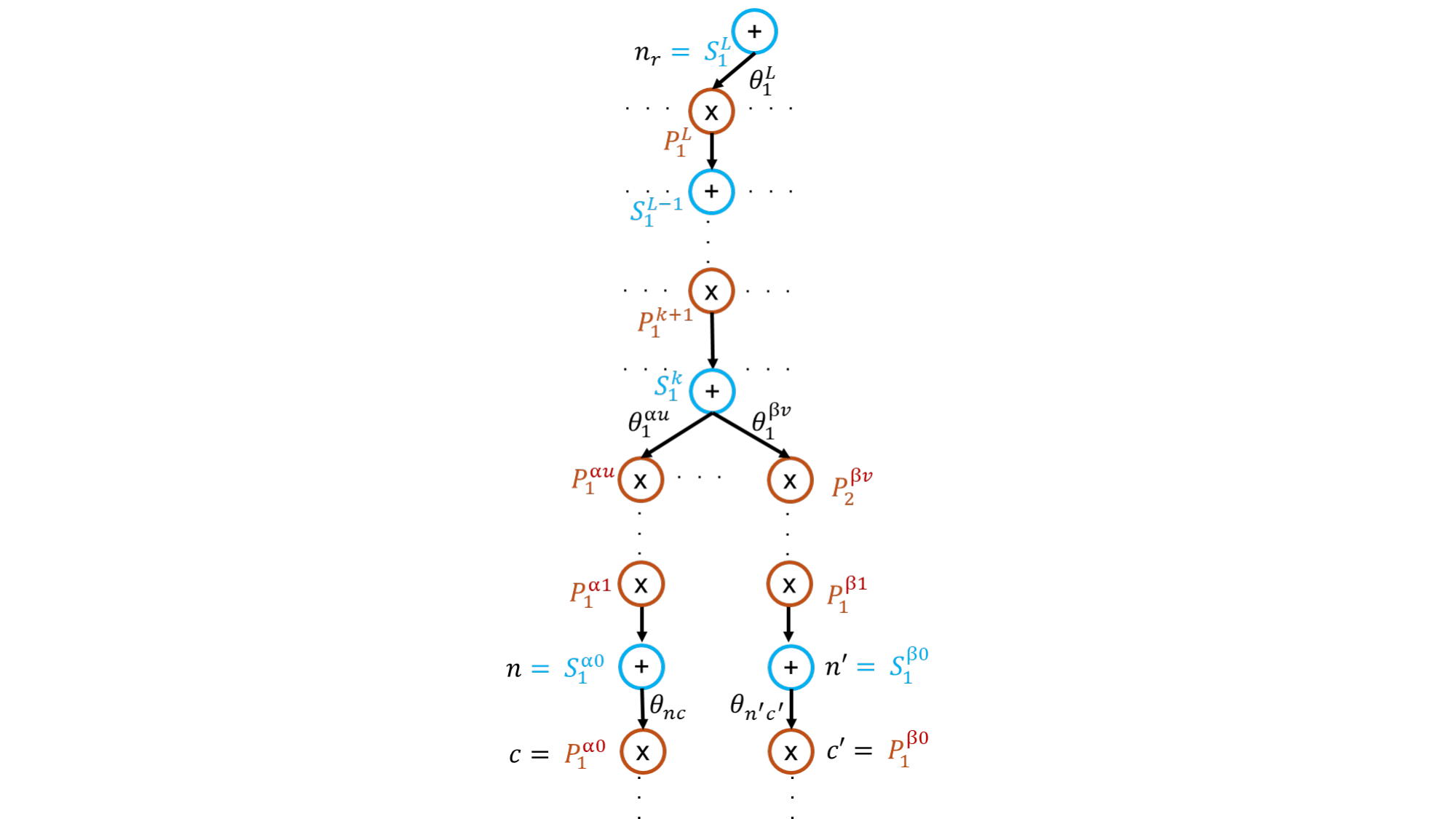}
        \caption{Sum pair}
        \label{fig:sumpair}
    \end{subfigure}
    \hfill
    \begin{subfigure}[b]{0.3\linewidth}
        \centering
        \includegraphics[width=\linewidth, trim=33em 0em 33em 0em, clip]{ 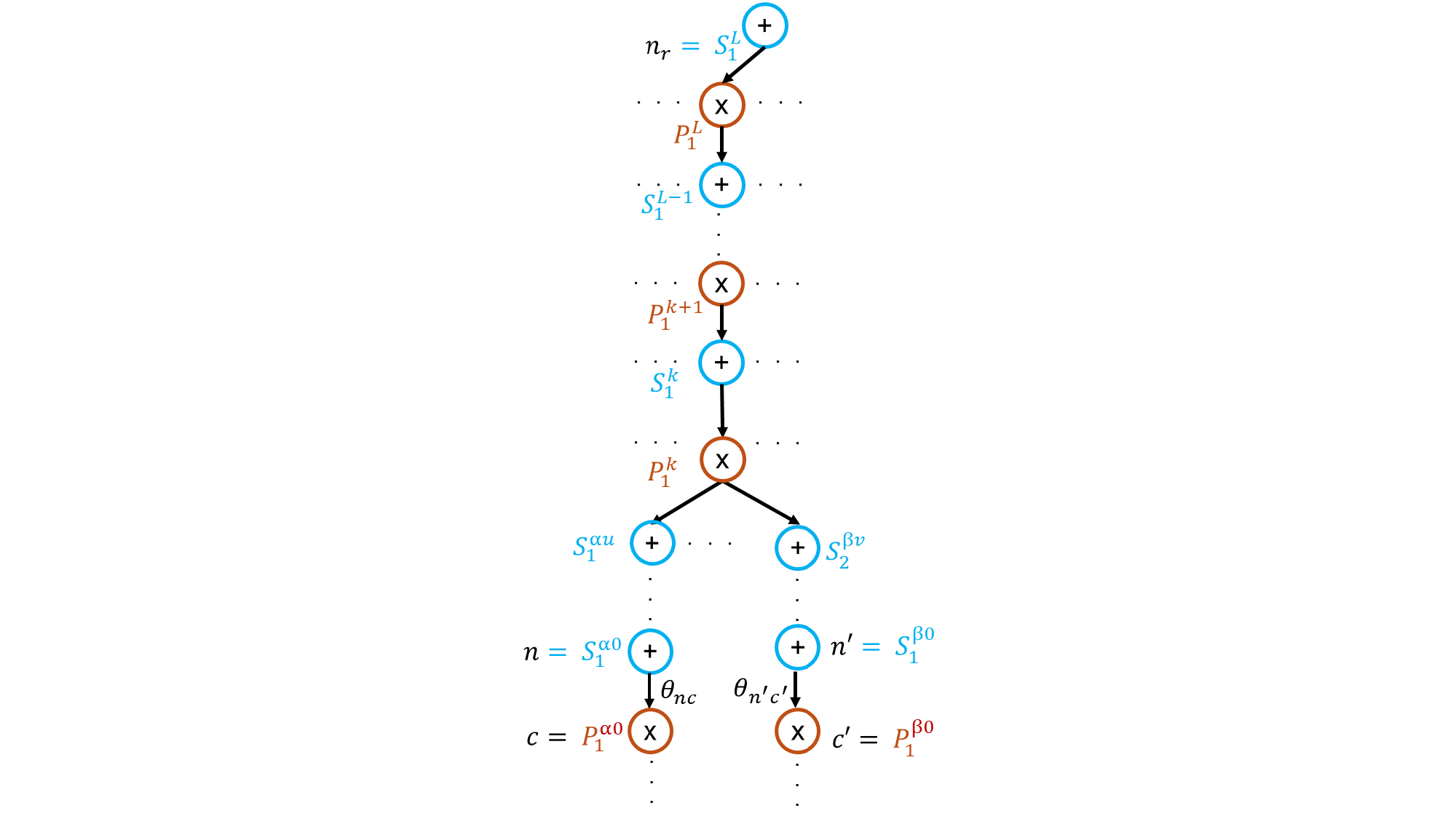}
        \caption{Product pair}
        \label{fig:productpair}
    \end{subfigure}
    \hfill
    \begin{subfigure}[b]{0.3\linewidth}
        \centering
        \includegraphics[width=\linewidth, trim=21em 47em 23em 0em, clip]{ 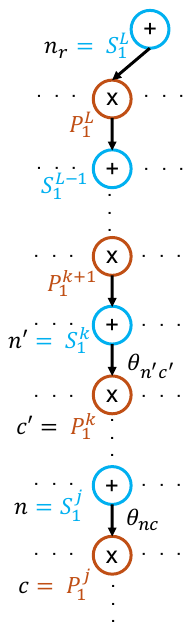}
        \caption{Path pair}
        \label{fig:pathpair}
    \end{subfigure}

    \caption{(a) The unique path from the root $n_r$ to a sum node $n$ in a Tree‐Structured PC (TS‐PC).  (b–d) The three possible relationships between a pair of edges $(\theta_{nc}, \theta_{n'c'})$—sum pair, product pair, and path pair—used in our Hessian derivation.}
    \label{fig:gradientderivation}
\end{figure}

\begin{lemma}
Consider a TS-PC path from the root to the sum node $n$ as depicted in the Figure \ref{fig:ts-pc-path-structure}. Then, the flow of $n$ is given by
\begin{equation}
F_n(x) = \theta^1_1 \dfrac{P^1_1(x)}{P_{n_{r}}(x)}\prod_{j=2}^L\theta_1^j\bar{P}^j_{11}(x)
\end{equation}
\label{lemma:flow-unroll}
\end{lemma}
\begin{proof}
In a tree-structured PC, every non-root node has a unique path from the root.  Let 
\[
n_r \;\to\; P^L_1 \;\to\; S^{L-1}_1\;\to\;\cdots\;\to\;P^1_1\;\to\;S^0_1 = n
\]
be such a canonical path to a sum node $n=S^0_1$, assuming $n$ to be at level $0$ and the root node $n_r$ to be at level $L$ as illustrated in Figure ~\ref{fig:ts-pc-path-structure}. For notational simplicity we have labeled every sum and product‐node along our canonical path by index “$1$”.  In a general TS‑PC path, the nodes at layer \(l\) would carry their own index \(i_l\), but the same results and derivations hold by replacing each “$1$” with the appropriate \(i_l\).

\noindent We can unroll the flow at $n$ using the notion of product complements of the product nodes along the path from $n$ to $n_r$.
By definition of a flow,
\begin{align*}
    F_n(x) = & F_c(x) = F_{P^1_1}(x) \hspace{3em} (\text{since} P_1^1 \ \text{is the unique parent of } n)\\
    = & \theta^1_1 \dfrac{P^1_1(x)}{S^1_1(x)}F_{S^1_1}(x)\\
    = & \theta^1_1 \dfrac{P^1_1(x)}{S^1_1(x)}\theta^2_1 \dfrac{P^2_1(x)}{S^2_1(x)} \ldots \theta^k_1\dfrac{P^k_1(x)}{S^k_1(x)} F_{S^k_1}(x) \\
    = & \theta^1_1 \dfrac{P^1_1(x)}{S^1_1(x)}\theta^2_1 \dfrac{\prod_{j\in in(P^2_1)}S^1_j(x)}{S^2_1(x)} \ldots \theta^k_1\dfrac{\prod_{j\in in(P^k_1)}S^{k-1}_j(x)}{S^k_1(x)} F_{S^k_1}(x) \\
    = & \theta^1_1\dfrac{P^1_1(x)}{S^k_1(x)}\prod_{l=2}^k\theta^l_1 \bar{P}^{l}_{11}(x) F_{S^k_1}(x)\\
    = & \theta^1_1 \dfrac{P^1_1(x)}{S^L_1(x)}\prod_{l=2}^L\theta^l_1 \bar{P}^{l}_{11}(x) F_{S^L_1}(x)\\
    = & \theta^1_1\dfrac{P^1_1(x)}{P_{n_r}(x)}\prod_{l=2}^L\theta^l_1 \bar{P}^{l}_{11}(x)     
\end{align*}
\end{proof}
\begin{corollary}
Under the same TS-PC path structure, the flow associated with the edge $(nc)$ that is parametrized by $\theta_{nc}$ (=$\theta^0_1$ in Figure \ref{fig:ts-pc-path-structure}) is given by
\begin{equation}
    F_{nc}(x) = \theta_{nc} \dfrac{P_c(x)}{P_{n_r}(x)}\prod_{l=1}^L\theta^l_1 \bar{P}^{l}_{11}(x)   
\end{equation}
and hence, the gradient of the likelihood at the root node with respect to the parameter $\theta^n_c$ can be written as
\begin{align}
    \dfrac{\partial P_{n_r}(x)}{\partial \theta_{nc}} = & P_c(x) \prod_{l=1}^L\theta^l_1 \bar{P}^{l}_{11}(x)
\end{align}
\label{cor:derivative-in-terms-of-product-complement}
\end{corollary}
\begin{proof}
The proof is a straight forward application of Lemma \ref{lemma:flow-unroll} and the flow defined on an edge: \[F_{nc}(x)=\theta_{nc}\,.\left(\dfrac{\partial \log P_{n_r}(x)}{\partial \theta_{nc}}\right)\]
\end{proof}

\noindent The key insight from the above corollary is that the derivative of the likelihood with respect to a mixing coefficient contains only weights and product complements at every layer; i.e., it does not include the output of the sum nodes along the path from the root to the node $n$.
Let $\theta_{nc}$ and $\theta_{n'c'}$ be the weights associated with two distinct edges in a TS-PC. Then, the pair $(\theta_{nc}, \theta_{n'c'})$ belongs to one of the three following pairs as illustrated in Figure \ref{fig:gradientderivation}:
\begin{itemize}
    \item \textbf{Sum pair} (Figure \ref{fig:sumpair}): If the deepest common ancestor from the root node is a sum node.
    \item \textbf{Product pair} (Figure \ref{fig:productpair}): If the deepest common ancestor from the root node is a product node.
    \item \textbf{Path pair} (Figure \ref{fig:pathpair}): If the edges corresponding to $\theta_{nc}$ and $\theta_{n'c'}$ lies on the same path from the root to a terminal node of the TS-PC.
\end{itemize}

\noindent Now, depending on how the two edges relate in the tree, we show that the mixed second partial derivatives take one of the following three closed forms.

\begin{theorem}
    \label{thm:Likelihood-Hessian}
    Let  $\theta_{nc},\theta_{n'c'}$ be two distinct sum-node weights in a TS-PC. Then, the double derivative of the likelihood of a PC with respect to $(\theta_{nc},\theta_{n'c'})$ can be defined as:
\[
    \dfrac{\partial^2 P_{nr}(x)}{\partial \theta_{n'c'} \partial \theta_{nc}} = \begin{cases}
     0,& \text{if } \theta_{nc} \text{ and } \theta_{n'c'}\text{ are sum pairs}\\
\dfrac{\dfrac{\partial P_{n_r}(x)}{\partial \theta_{nc}}\dfrac{\partial P_{n_r}(x)}{\partial \theta_{n'c'}}}{\theta^k_1 P^k_{1}(x)\left(\prod_{l=k+1}^L\theta^l_1 \bar{P}^{l}_{11}(x)\right)} ,& \text{if } \theta_{nc} \text{ and } \theta_{n'c'}\text{ are product pairs and $P_{1}^k$ denotes the common product node}\\
    \dfrac{1}{\theta_{n'c'}} \cdot \dfrac{\partial P_{n_r}(x)}{\partial \theta_{nc}},& \text{if } \theta_{nc} \text{ and } \theta_{n'c'}\text{ are path pairs where $\theta_{n'c'}$ is closer to the root node}
\end{cases}
\]    
\end{theorem}

\begin{proof}
Each case follows by differentiating the product form from Corollary~\ref{cor:derivative-in-terms-of-product-complement} once more and observing which factors depend on which parameter. More formally, from Corollary \ref{cor:derivative-in-terms-of-product-complement} we have,
    \[
    \dfrac{\partial P_{n_r}(x)}{\partial \theta_{nc}} =  P_c(x) \prod_{l=1}^L\theta^l_1 \bar{P}^{l}_{11}(x)
\]
\begin{itemize}
    \item \textbf{Sum Pair:} Let $S^k_1$ be the deepest common ancestor of $\theta_{nc}$ and $\theta_{n'c'}$. As illustrated in Figure \ref{fig:sumpair}, let us use $\alpha i$ to denote the level of nodes in the path from $S^k_1$ to $n$. Then, the derivative can be written as 
\[
    \dfrac{\partial P_{n_r}(x)}{\partial \theta_{nc}} = \underbrace{P_c(x) \left(\prod_{l=\alpha 1}^{\alpha u}\theta^l_1 \bar{P}^{l}_{11}(x) \right)}_{\text{Term 1}}\underbrace{\left(\prod_{l=k}^L\theta^l_1 \bar{P}^{l}_{11}(x)\right)}_{\text{Term 2}}
\]
Clearly, Term 1 consists only the parameters and product complements in the path from $n$ to $S^k_1$. Therefore, it is independent of $\theta_{n'c'}$. Term 2 consists of product complements from $S^k_1$ to the root node, in which $S^k_1(x)$ is absent, making it also independent of $\theta_{n'c'}$. Thus,
\[
\dfrac{\partial^2 P_{nr}(x)}{\partial \theta_{nc} \partial \theta_{n'c'}} = 0
\]

\item \textbf{Product Pair:} Let $P^k_1$ be the deepest common ancestor of $\theta_{nc}$ and $\theta_{n'c'}$. As illustrated in Figure \ref{fig:productpair} let us use $\alpha i$ to denote the level of nodes in the path from $P^k_1$ to $n$ and $\beta i$ to denote the level of nodes in the path from $P^k_1$ to $n'$. Then we can split the partial derivative as 
\[    \dfrac{\partial P_{n_r}(x)}{\partial \theta_{nc}} = \underbrace{P_c(x) \left(\prod_{l=\alpha 1}^{\alpha u}\theta^l_1 \bar{P}^{l}_{11}(x) \right)}_\text{Term \emph{$n$ to $k-1$}}\underbrace{\theta^k_1\bar{P}^k_{11}(x)}_{\text{Term $k$}} \underbrace{\left(\prod_{l=k+1}^L\theta^l_1 \bar{P}^{l}_{11}(x)\right)}_{\text{Term \emph{$k+1$ to root}}}\]
Note that in a product complement, the output of the sum node along the path is absent. Therefore, the output of the sum nodes along the path from level $k+1$ to the root is absent in the Term \emph{$k+1$ to root}. Thus, $\bar{P}^k_{11}$ is the only term that depends on $\theta_{n'c'}$. Thus, 
\begin{align}
    \dfrac{\partial^2 P_{n_r}(x)}{\partial \theta_{n'c'}\partial \theta_{nc}} 
    & = P_c(x) \left(\prod_{l=\alpha 1}^{\alpha u}\theta^l_1 \bar{P}^{l}_{11}(x) \right) \left(\prod_{l=k+1}^L\theta^l_1 \bar{P}^{l}_{11}(x)\right) \theta^k_1\dfrac{\partial\bar{P}^k_{11}(x)}{\partial \theta_{n'c'}}\\
    & = P_c(x) \left(\prod_{l=\alpha 1}^{\alpha u}\theta^l_1 \bar{P}^{l}_{11}(x) \right) \left(\prod_{l=k+1}^L\theta^l_1 \bar{P}^{l}_{11}(x)\right) \theta^k_1 \bar{\bar{P}}^l_{112}(x) \left(\prod_{l=\beta 1}^{\beta v}\theta^l_1 \bar{P}^l_{11}\right)P_{c'}(x)\\
    & = \theta^k_1 \left(\prod_{l=k+1}^L\theta^l_1 \bar{P}^{l}_{11}(x)\right) \left(\prod_{l=\alpha 1}^{\alpha u}\theta^l_1 \bar{P}^{l}_{11}(x) \right) \left(\prod_{l=\beta 1}^{\beta v}\theta^l_1 \bar{P}^l_{11}\right) \bar{\bar{P}}^l_{112}(x) P_c(x) P_{c'}(x)
\end{align}

\noindent The partial derivatives of $P_{n_r}(x)$ with respect to $\theta_{nc}$ and $\theta_{n'c'}$ can be written as:

\[    \dfrac{\partial P_{n_r}(x)}{\partial \theta_{nc}} = \underbrace{P_c(x) \left(\prod_{l=\alpha 1}^{\alpha u}\theta^l_1 \bar{P}^{l}_{11}(x) \right)}_\text{Term \emph{$n$ to $k-1$}}\underbrace{\theta^k_1\bar{P}^k_{11}(x)}_{\text{Term $k$}} \underbrace{\left(\prod_{l=k+1}^L\theta^l_1 \bar{P}^{l}_{11}(x)\right)}_{\text{Term k+1 to root}}\]

\[    \dfrac{\partial P_{n_r}(x)}{\partial \theta_{n'c'}} = \underbrace{P_c'(x) \left(\prod_{l=\beta 1}^{\beta v}\theta^l_1 \bar{P}^{l}_{11}(x) \right)}_\text{Term n' to k-1}\underbrace{\theta^k_1\bar{P}^k_{12}(x)}_{\text{Term k}} \underbrace{\left(\prod_{l=k+1}^L\theta^l_1 \bar{P}^{l}_{11}(x)\right)}_{\text{Term k+1 to root}}\]
\text{Thus, we can simplify the double derivative above as: }\[\dfrac{\partial^2 P_{n_r}(x)}{\partial \theta_{n'c'}\partial \theta_{nc}}  = \dfrac{\dfrac{\partial P_{n_r}(x)}{\partial \theta_{nc}}\dfrac{\partial P_{n_r}(x)}{\partial \theta_{n'c'}}}{\theta^k_1 P^k_{1}(x)\left(\prod_{l=k+1}^L\theta^l_1 \bar{P}^{l}_{11}(x)\right)}\]
Note that for both sum pairs and product pairs, we do not need to consider the order of derivatives, as their formulas are symmetric (i.e., interchanging $\theta_{nc}$ and $\theta_{n'c'}$ does not affect the result). However, for path pairs, this symmetry is absent, so we will show that the second derivative yields the same value regardless of the order of differentiation.
\\
\item \textbf{Path Pair:}
Without loss of generality, Let $n'$ be the node nearest to the root.
Then we can divide the derivative as
\[
    \dfrac{\partial P_{n_r}(x)}{\partial \theta_{nc}} = P_c(x) \left(\prod_{l=j}^{k-1}\theta^l_1 \bar{P}^{l}_{11}(x) \right)\theta_{n'c'}\bar{P}^k_{11}(x)\left(\prod_{l=k+1}^L\theta^l_1 \bar{P}^{l}_{11}(x)\right)
\]
Clearly the derivative is a linear function of $\theta_{n'c'}$, so taking its derivative results in 

\begin{align}
    \dfrac{\partial^2 P_{n_r}(x)}{\partial \theta_{n'c'}\partial \theta_{nc}}
    &= P_c(x) \left(\prod_{l=j}^{k-1}\theta^l_1 \bar{P}^{l}_{11}(x) \right)\bar{P}^k_{11}(x)\left(\prod_{l=k+1}^L\theta^l_1 \bar{P}^{l}_{11}(x)\right)\\
    &= \dfrac{1}{\theta_{n'c'}}\dfrac{\partial P_{n_r}(x)}{\partial \theta_{nc}}
\end{align}

Now, let us consider the case when the order of derivation is reversed. Then the derivative is
    \[
    \dfrac{\partial P_{n_r}(x)}{\partial \theta_{n'c'}} =   P_{c'}(x) \prod_{l=k+1}^L\theta^l_1 \bar{P}^{l}_{11}(x) 
\]
Here $P_{c'}(x)$ is the only term dependent on $\theta_{nc}$. 
So,   
\begin{align}
    \dfrac{\partial^2 P_{n_r}(x)}{\partial \theta_{nc}\partial \theta_{n'c'}} 
    &=  \left(\prod_{l=k+1}^L\theta^l_1 \bar{P}^{l}_{11}(x)\right) \dfrac{\partial P_{c'}(x)}{\partial \theta_{nc}} \nonumber \\
    &= \left(\prod_{l=k+1}^L\theta^l_1 \bar{P}^{l}_{11}(x)\right) \bar{P}^k_{11}(x)\left(\prod_{l=j}^{k-1}\theta^l_1 \bar{p}^l_{11}(x)\right) P_c(x) \nonumber \\
    &= \dfrac{1}{\theta_{n'c'}}\dfrac{\partial P_{n_r}(x)}{\partial \theta_{nc}}
\end{align}

\end{itemize}
\end{proof}

\noindent Now that we have established the closed form expressions for the double derivative of the likelihood function, we can proceed to derive the double derivative of \emph{log‑likelihood}.
\begin{proposition}
If $\theta_{nc}$ and $\theta_{n'c'}$ is a sum pair, then
\begin{equation}
\dfrac{\partial^2 \log P_{n_r}(x)}{\partial \theta_{n'c'} \partial \theta_{nc}} = -\left(\dfrac{F_{nc}(x)}{\theta_{nc}}\right) \left(\dfrac{F_{n'c'}(x)}{\theta_{n'c'}}\right)
\end{equation}
In other words, the mixed double derivative of the log-likelihood with respect to the sum pair parameters is the negative product of the derivative of the log-likelihood with respect to the individual parameters.
\end{proposition}
\begin{proof}
We have,
    \[\dfrac{\partial log P_{n_r} (x)}{\partial \theta_{nc}} = \dfrac{1}{P_{n_r}(x)}\:\dfrac{\partial P_{n_r} (x)}{\partial \theta_{nc}} \]
Now, by applying the product rule,  
    \begin{align*}
        \dfrac{\partial^2 log P_{n_r} (x)}{\partial \theta_{n'c'}\theta_{nc}} & = \dfrac{1}{P_{n_r}(x)}\:\dfrac{\partial^2 P_{n_r} (x)}{\partial \theta_{nc}\theta_{n'c'}} - 
    \dfrac{\partial P_{n_r} (x)}{\partial \theta_{nc}} \dfrac{\partial P_{n_r} (x)}{\partial \theta_{n'c'}} \left(\dfrac{1}{P_{n_r}(x)}\right)^2\\
    & = - \dfrac{\partial P_{n_r} (x)}{\partial \theta_{nc}} \dfrac{\partial P_{n_r} (x)}{\partial \theta_{n'c'}}\left(\dfrac{1}{P_{n_r}(x)}\right)^2\\
    & = - \dfrac{\partial log P_{n_r} (x)}{\partial \theta_{nc}} \dfrac{\partial log P_{n_r} (x)}{\partial \theta_{n'c'}}\\
    & = -\left(\dfrac{F_{nc}(x)}{\theta_{nc}}\right) \left(\dfrac{F_{n'c'}(x)}{\theta_{n'c'}}\right)
    \end{align*}
where we use the fact from Theorem \ref{thm:Likelihood-Hessian} that  $\dfrac{\partial^2 P_{nr}(x)}{\partial \theta_{n'c'} \partial \theta_{nc}} = 0$ and the relation between Flows and log-likelihood derivative.
  
\noindent Furthermore, as a special case when $\theta_{n'c'}=\theta_{nc}$, we get
\begin{equation}
\dfrac{\partial^2 \log P_{n_r}(x)}{\partial^2 \theta_{nc}} = -\left(\dfrac{F_{nc}(x)}{\theta_{nc}}\right)  ^2    
\end{equation}
\end{proof}
\begin{proposition}
    If $\theta_{nc}$ and $\theta_{n'c'}$ is a product pair, then
\begin{equation}
 \dfrac{\partial^2 log P_{n_r} (x)}{\partial \theta_{n'c'}\theta_{nc}} = \dfrac{1}{P_{n_r}(x)}\:\dfrac{\dfrac{\partial P_{n_r}(x)}{\partial \theta_{nc}}\dfrac{\partial P_{n_r}(x)}{\partial \theta_{n'c'}}}{\theta^k_1 P^k_{1}(x)\left(\prod_{l=k+1}^L\theta^l_1 \bar{P}^{l}_{11}(x)\right)} - \dfrac{\partial \log P_{n_r}(x)}{\partial \theta_{nc}}\dfrac{\partial \log P_{n_r}(x)}{\partial \theta_{n'c'}}
\end{equation}
\end{proposition}
\begin{proof}
We have,
    \begin{align*}
        \dfrac{\partial^2 log P_{n_r} (x)}{\partial \theta_{n'c'}\theta_{nc}} &= \dfrac{1}{P_{n_r}(x)}\:\dfrac{\partial^2 P_{n_r} (x)}{\partial \theta_{n'c'}\theta_{nc}} \\&- 
    \dfrac{\partial P_{n_r} (x)}{\partial \theta_{nc}} \dfrac{\partial P_{n_r} (x)}{\partial \theta_{n'c'}} \left(\dfrac{1}{P_{n_r}(x)}\right)^2
    \end{align*}

    \noindent From Theorem \ref{thm:Likelihood-Hessian} we have, $$  \dfrac{\partial^2 log P_{n_r} (x)}{\partial \theta_{n'c'}\theta_{nc}}= \dfrac{\dfrac{\partial P_{n_r}(x)}{\partial \theta_{nc}}\dfrac{\partial P_{n_r}(x)}{\partial \theta_{n'c'}}}{\theta^k_1 P^k_{1}(x)\left(\prod_{l=k+1}^L\theta^l_1 \bar{P}^{l}_{11}(x)\right)} $$

\noindent Therefore,
    \begin{align*}
        \dfrac{\partial^2 log P_{n_r} (x)}{\partial \theta_{n'c'}\theta_{nc}} &= \dfrac{1}{P_{n_r}(x)}\:\dfrac{\dfrac{\partial P_{n_r}(x)}{\partial \theta_{nc}}\dfrac{\partial P_{n_r}(x)}{\partial \theta_{n'c'}}}{\theta^k_1 P^k_{1}(x)\left(\prod_{l=k+1}^L\theta^l_1 \bar{P}^{l}_{11}(x)\right)}\\&- \dfrac{\partial \log P_{n_r}(x)}{\partial \theta_{nc}}\dfrac{\partial \log P_{n_r}(x)}{\partial \theta_{n'c'}}\\
        &=\dfrac{P_{n_r}(x)\dfrac{ \partial log P_{n_r}(x)}{\partial \theta_{nc}}\dfrac{\partial log P_{n_r}(x)}{\partial \theta_{n'c'}}}{\theta^k_1 P^k_{1}(x)\left(\prod_{l=k+1}^L\theta^l_1 \bar{P}^{l}_{11}(x)\right)} - \dfrac{\partial \log P_{n_r}(x)}{\partial \theta_{nc}}\dfrac{\partial \log P_{n_r}(x)}{\partial \theta_{n'c'}}\\
        &=\dfrac{P_{n_r}(x)\dfrac{F_{nc}(x)}{\theta_{nc}} \dfrac{F_{n'c'}(x)}{\theta_{n'c'}}}{\theta^k_1 P^k_{1}(x)\left(\prod_{l=k+1}^L\theta^l_1 \bar{P}^{l}_{11}(x)\right)} - \dfrac{F_{nc}(x)}{\theta_{nc}} \dfrac{F_{n'c'}(x)}{\theta_{n'c'}}
    \end{align*}
    
\end{proof}
\begin{proposition}
    If $\theta_{nc}$ and $\theta_{n'c'}$ is a path pair, with $\theta_{n'c'}$ closer to the root node, then
\begin{equation}
\dfrac{\partial^2 \log P_{n_r}(x)}{\partial \theta_{n'c'} \partial \theta_{nc}} = \dfrac{\partial^2 \log P_{n_r}(x)}{\partial \theta_{nc} \partial \theta_{n'c'}} = \dfrac{1}{\theta_{n'c'}} \cdot \left(\dfrac{F_{nc}(x)}{\theta_{nc}}\right) - \left(\dfrac{F_{nc}(x)}{\theta_{nc}}\right)\left(\dfrac{F_{n'c'}(x)}{\theta_{n'c'}}\right)
\end{equation}
\end{proposition}
\begin{proof}
We have,
    \[\dfrac{\partial^2 log P_{n_r} (x)}{\partial \theta_{n'c'}\theta_{nc}} = \dfrac{1}{P_{n_r}(x)}\:\dfrac{\partial P_{n_r} (x)}{\partial \theta_{n'c'}\theta_{nc}} - \dfrac{\partial P_{n_r} (x)}{\partial \theta_{nc}} \dfrac{\partial P_{n_r} (x)}{\partial \theta_{n'c'}} \left(\dfrac{1}{P_{n_r}(x)}\right)^2\]
    From Theorem \ref{thm:Likelihood-Hessian} we have, $\dfrac{\partial^2 P_{nr}(x)}{\partial \theta_{n'c'} \partial \theta_{nc}} = \dfrac{1}{\theta_{n'c'}} \cdot \dfrac{\partial P_{n_r}(x)}{\partial \theta_{nc}} $\\  
Therefore,
\[\dfrac{\partial^2 log P_{n_r} (x)}{\partial \theta_{n'c'}\theta_{nc}} = \dfrac{1}{P_{n_r}(x)}\: \dfrac{1}{\theta_{n'c'}} \cdot \dfrac{\partial P_{n_r}(x)}{\partial \theta_{nc}}-  \dfrac{\partial P_{n_r} (x)}{\partial \theta_{nc}} \dfrac{\partial P_{n_r} (x)}{\partial \theta_{n'c'}} \left(\dfrac{1}{P_{n_r}(x)}\right)^2\]
    \[ = \dfrac{1}{\theta_{n'c'}} \cdot \dfrac{\partial log P_{n_r}(x)}{\partial \theta_{nc}}-  \dfrac{\partial log P_{n_r} (x)}{\partial \theta_{nc}} \dfrac{\partial log P_{n_r} (x)}{\partial \theta_{n'c'}}\]
\[ = \dfrac{1}{\theta_{n'c'}} \cdot \left(\dfrac{F_{nc}(x)}{\theta_{nc}}\right) - \left(\dfrac{F_{nc}(x)}{\theta_{nc}}\right)\left(\dfrac{F_{n'c'}(x)}{\theta_{n'c'}}\right)\]    
\end{proof}
\subsection{Tractable Hessian Trace Computation for a General (Non-Tree Structured) PCs}
\begin{proposition}
The diagonal entry of the Hessian of the log-likelihood with respect to $\theta_{nc}$ for a general PC is given by:
\[
\dfrac{\partial^2 \log P_{n_r}(x)}{\partial \theta_{nc}^2} = - \left( \dfrac{\partial \log P_{n_r}(x)}{\partial \theta_{nc}} \right)^2.
\]
\end{proposition}

\begin{proof}
As $P_{n_r}(x)$ is linear in $\theta_{nc}$, its second derivative with respect to $\theta_{nc}$ is zero. Applying the chain rule, we obtain:
\[
\dfrac{\partial \log P_{n_r}(x)}{\partial \theta_{nc}} = \dfrac{1}{P_{n_r}(x)} \cdot \dfrac{\partial P_{n_r}(x)}{\partial \theta_{nc}}.
\]
Differentiating again and using the fact that \( \dfrac{\partial^2 P_{n_r}(x)}{\partial \theta_{nc}^2} = 0 \), we get:
\[
\dfrac{\partial^2 \log P_{n_r}(x)}{\partial \theta_{nc}^2}
= - \left( \dfrac{\partial \log P_{n_r}(x)}{\partial \theta_{nc}} \right)^2 = -\;\Bigl(\dfrac{F_{nc}(x)}{\theta_{nc}}\Bigr)^{2}.
\]
\end{proof}
\noindent Thus, the Hessian trace is given by:
\[
\mathrm{Tr}\bigl(\nabla^2\log P_{n_r}(x)\bigr)
\;=\;
-\;\sum_{(n,c)}\;
\Bigl(\dfrac{F_{nc}(x)}{\theta_{nc}}\Bigr)^{2},
\]
which can be computed in one forward–backward pass in $O(|P|\,|D|)$ time using edge flows as outlined in Algorithm \ref{alg:hessian-trace}.


\begin{algorithm}[H]
\caption{\textbf{Exact Hessian Trace Computation for General PCs using Edge-Flows}}
\begin{algorithmic}[1]
\Require 
  Probabilistic circuit $\mathsf{PC}$ with edges $E$, parameters $P=\{\theta_{nc}\}_{(n,c)\in E}$;
  Dataset $D = \{x^{(i)}\}_{i=1}^N$
\Ensure 
  $|\mathrm{Tr}\bigl(\nabla^2 \log P_{\mathsf{PC}}(D)\bigr)|$
\State\textbf{Initialize:} $\mathrm{\textit{abs\_trace}} \gets 0$
\For{each data point $x$ in $D$}                                       \Comment{\textcolor{gray}{$N$ iterations}}
  \State \emph{\textbf{Forward pass:}} compute node outputs 
         $p_n(x)\;\forall$ nodes $n$ in $\mathsf{PC}$
         \Comment{\textcolor{gray}{Runs in time $O(|E|)$ = $O(|P|)$ }}
  \State \emph{\textbf{Backward pass (compute edge-flows):}}
    \State \quad Initialize all node‐flows $F_n \gets 0$, and edge‐flows $F_e\gets 0$
    \State \quad Set $F_{\mathsf{root}}\gets 1$
    \For{each node $n$ in topological order from root to leaves}       \Comment{\textcolor{gray}{Equals making one pass over all edges}}
      \If{$n$ is a sum node}
        \For{each child edge $e=(n\!\to\!c)$}
          \State $F_e \gets F_n \times \theta_e \times \tfrac{p_c(x)}{p_n(x)}$ \hspace{1.9em} \textcolor{gray}{// Compute Edge Flow}
          \State $F_c \gets F_c + F_e$ \hspace{5em} \textcolor{gray}{// Update Node Flow}
        \EndFor
      \ElsIf{$n$ is a product node}
        \For{each child edge $e=(n\!\to\!c)$}
          \State $\overline{P}_{n,c}(x)\gets p_n(x)/p_c(x)$  \Comment{\textcolor{gray}{Precomputed in the forward pass}}
          \State $F_e \gets F_n \times \overline{P}_{n,c}(x)$  \hspace{1.9em} \textcolor{gray}{// Compute Edge Flow}
          \State $F_c \gets F_c + F_e$ \hspace{4em} \textcolor{gray}{// Update Node Flow}
        \EndFor
      \EndIf
    \EndFor                                                        \Comment{\textcolor{gray}{Runs in time $O(|E|) = O(|P|)$}}
  \State \emph{\textbf{Accumulate trace contribution:}}
    \For{each edge $e\in E$}
      \State $\mathrm{\textit{abs\_trace}} \mathrel{+}= \bigl(F_e/\theta_e\bigr)^2$
    \EndFor                                                        \Comment{\textcolor{gray}{Runs in time $O(|E|) = O(|P|)$}}
\EndFor                                                            \Comment{Thus, total over the $N$ points, complexity = $O(|P||D|)$}
\State \Return $\mathrm{\textit{abs\_trace}}$
\end{algorithmic}
\label{alg:hessian-trace}
\end{algorithm}
    
\subsection{Sharpness-Aware Regularization for PCs}
In this section we show how to incorporate our tractable Hessian‐trace penalty into the EM updates for sum-node parameters.  Recall that at each sum node \(n\), vanilla EM maximizes
\[
\mathcal{L}_n(\theta_{n\cdot})
\;=\;
\sum_{c\in\mathrm{ch}(n)}F_{nc}(x)\,\log\theta_{nc}
\quad\text{s.t.}\quad
\sum_{c}\theta_{nc}=1,
\]
where \(F_{nc}(x)\) denotes the edge-flows (“soft counts”).

\begin{proposition}
    The EM update for a parameter \(\theta_{nc}\) associated with a sum node \(n\), under a Hessian trace sharpness regularizer is the solution to the cubic equation:
    \begin{equation}
        \lambda\,  \theta_{nc}^3 - F_{nc}(x)\, \theta_{nc}^2 - 2\,\mu F_{nc}(x) = 0,
    \end{equation}
    where \(F_{nc}(x)\) is the expected edge-flow along edge $(n.c)$, and \(\lambda\), \(\mu\) are the Lagrange multipliers for the normalization and trace constraints, respectively.
\end{proposition}
\begin{proof}
To extend vanilla EM with our Hessian trace regularizer and discourage sharp optima, we further update the objective $\mathcal{L}_n$ by \emph{constraining} the Hessian trace,
\[
\sum_{c\in\mathrm{ch}(n)}\Bigl(\frac{F_{nc}(x)}{\theta_{nc}}\Bigr)^{2}
\;\le\; m,
\]
which upper-bounds the curvature at \(n\).  Thus, for each parameter $\theta_{n,c}$  we now solve the regularized objective:

\[
\begin{aligned}
\theta^*_{n\cdot} = \arg\max_{\theta_{n\cdot}} \quad & \sum_{c \in \mathrm{ch}(n) }F_{nc}(x)\: \log \,\theta_{nc} \\
\text{subject to} \quad & \sum\limits_{c \in \mathrm{ch}(n)} \theta_{nc} = 1, \\
                        & \sum\limits_{c \in \mathrm{ch}(n)} \left( \dfrac{F_{nc}(x)}{\theta_{nc}} \right)^2 \leq m
\end{aligned}
\]

\noindent The Lagrangian formulation for the above constrained maximization objective can be written as: 
\[ \mathcal{L}(\theta_{n\cdot}, \lambda, \mu) = \sum\limits_{c \in \mathrm{ch}(n) }F_{nc}(x)\: \log \,\theta_{nc} - \lambda \left( \sum\limits_{c \in \mathrm{ch}(n)} \theta_{nc} - 1\right) - \mu \left( \sum\limits_{c \in \mathrm{ch}(n)} \left(\dfrac{F_{nc}(x)}{\theta_{nc}} \right)^2 - m \right)\]

\noindent Taking the partial derivative with respect to $\theta_{nc}$ and setting to $0$, we get 

\[\dfrac{\partial\mathcal{L}}{\partial\theta_{nc}} = F_{nc}(x)\, \dfrac{1}{\theta_{nc}} - \lambda + \mu \dfrac{2 F_{nc}(x)^2}{\theta_{nc}^3} = 0\]
Thus resulting in the cubic equation, (assuming $\theta_{nc}>0)$,

\[ \lambda \, \theta_{nc}^3 -F_{nc}(x)\, \theta_{nc}^2 -  2\mu F_{nc}(x)^2  = 0\]

\end{proof}

\noindent Solving a cubic at every node and every iteration can be costly and numerically unstable.  Fortunately, because our Hessian-trace penalty is itself a sum of \((F_{nc}/\theta_{nc})^{2}\), one can equivalently \emph{bound} it by directly bounding the gradients
\[
\dfrac{\partial \log P_{n_r}(x)}{\partial \theta_{nc}} = \frac{F_{nc}(x)}{\theta_{nc}}\;\le\; r,
\]
which leads to the following \emph{quadratic} update.
\begin{theorem}
The EM update for a parameter $\theta_{nc}$ associated with a sum node \(n\), regularized by a constraint on the gradient, is given by:
\[
\theta_{nc} = \dfrac{F_{nc}(x) + \sqrt{F_{nc}(x)^2 + 4 \lambda \mu F_{nc}(x)}}{2 \lambda},
\]
where \(F_{nc}(x)\) denotes the flow along the edge from sum node \(n\) to its child \(c\), and \(\lambda, \mu \ge 0\) are the Lagrange multipliers corresponding to the normalization and regularization constraints, respectively.
\label{thm:Gradient update}
\end{theorem}

\begin{proof}

We now maximize the expected log-likelihood subject to the constraints that the parameters lie on the probability simplex and the simpler constraint that the local gradient is upper bounded. Thus the regularized objective can be expressed as:
\[
\begin{aligned}
\theta^*_{n\cdot} = \arg\max_{\theta_{n\cdot}} \quad & \sum_{c \in \mathrm{ch}(n) }F_{nc}(x)\: \log \,\theta_{nc} \\
\text{subject to} \quad & \sum\limits_{c \in \mathrm{ch}(n)} \theta_{nc} = 1, \\
                        & \sum\limits_{c \in \mathrm{ch}(n)}  \dfrac{F_{nc}(x)}{\theta_{nc}}  \leq m \hspace{5em} (\text{using the identity} \ F_{nc}(x) = \theta_{nc} \, \dfrac{\partial \log P_{n_r}(x)}{\partial \theta_{nc}})
\end{aligned}
\]


\noindent Forming the Lagrangian and applying KKT conditions leads to the quadratic equation:
\[
\lambda\, \theta_{nc}^2 - F_{nc}(x) \, \theta_{nc} - \mu F_{nc}(x) = 0.
\]
The discriminant of this equation is:
\[
F_{nc}(x)^2 + 4 \lambda \mu F_{nc}(x),
\]
which is nonnegative for \(F_{nc}(x), \lambda, \mu \ge 0\). Thus, complex roots cannot occur. As for the existence of multiple roots, observe that: $F_{nc}(x)^2 \le F_{nc}(x)^2 + 4 \lambda \mu F_{nc}(x) $, as $\lambda \ge 0$ and $\mu \ge 0$. Taking square roots:
\[
\sqrt{F_{nc}(x)^2} \le \sqrt{F_{nc}(x)^2 + 4 \lambda \mu F_{nc}(x)}
\]
\[
F_{nc}(x) - \sqrt{F_{nc}(x)^2 + 4 \lambda \mu F_{nc}} \le 0
\]
As we constrain the parameters to take only non-negative values, we can disregard the negative root of the equation. Hence, the unique parameter update equation is given by:
\[
\theta_{nc} = \dfrac{F_{nc}(x) + \sqrt{F_{nc}(x)^2 + 4 \lambda \mu F_{nc}(x)}}{2 \lambda}
\]
\end{proof}

\noindent Although our derivation in Theorem \ref{thm:Gradient update} is presented for the single-sample (“stochastic”) case, the same closed-form update extends immediately to mini-batch or full-batch EM by simply summing the expected edge‐flows $F_{nc}(x)$ over all samples in the batch (see lines 3–5 of Algorithm \ref{alg:sharpness-aware-em}). To reduce the variance introduced by small batches, we apply a running-average smoothing: 
\(\theta_{nc}^{new} = (1-\alpha)\theta_{nc}^{old} + \alpha\tilde{\theta}_{nc} \) where $\alpha \in [0,1]$ and $\tilde{\theta}_{nc}$ denotes the roots of the quadratic equation as given by Theorem \ref{thm:Gradient update}. In the KKT system, $\lambda$ appears only as a normalization multiplier and cannot be solved in closed form. In practice we simply fix $\lambda=1$ and then renormalize each sum-node’s parameters to lie on the simplex (see line 11 of Algorithm \ref{alg:sharpness-aware-em}).
 

\begin{algorithm}[H]
\caption{\textbf{Sharpness‐Aware EM for General PCs}}
\begin{algorithmic}[1]
\Require 
  Probabilistic circuit $\mathsf{PC}$ with sum‐node parameters 
  $P=\{\theta_{nc}\}_{(n,c)\in E}$; Dataset 
  $D=\{x^{(i)}\}_{i=1}^N$; Regularization weight $\mu$; Simplex Constraint weight $\lambda$; Smoothing factor $\alpha$; \#Epochs $E$
\Ensure 
  Updated sum node parameters $P$ under the sharpness regularized objective
\State \textbf{Initialize:} for each sum‐node $n$, set 
  $\theta_{n\cdot}$ uniformly on its simplex
\For{epoch in 1 \ldots E}                                                     \Comment{\textcolor{gray}{until convergence/max epoch E}}
  \State \emph{\textbf{E‐step:}} Compute expected edge flows

      \State Run forward–backward passes on $\mathsf{PC}$ over $D$
      \State Obtain $\{F_{nc}(x)\}_{(n,c)\in E,\,x\in D}$
                           \Comment{\textcolor{gray}{Runs in $O(|P||D|)$}}
  \State \emph{\textbf{M‐step:}} Apply sharpness‐aware update

      \ForAll{sum‐nodes $n$}                                        \Comment{\textcolor{gray}{independent per node}}
        \ForAll{child edges $(n\!\to\!c)$}
          \State $\tilde\theta_{nc} \gets \dfrac{\,\sum_{x}F_{nc}(x)\;+\;\sqrt{\bigl(\sum_{x}F_{nc}(x)\bigr)^2 \;+\;4\,\lambda\,\mu\,\sum_{x}F_{nc}(x)}}{2\,\lambda}$  
          \hfill \textcolor{gray}{// Theorem.~\ref{thm:Gradient update}}
        \EndFor
        \State Normalize: $\forall c \in \mathrm{ch}(n)$ \ set \  $\tilde\theta_{nc}\gets \tilde\theta_{nc}/\sum_{c \in \mathrm{ch}(n)}\tilde\theta_{nc}$
        \hfill \textcolor{gray}{// project onto simplex}
      \EndFor
      \State $\theta \gets (1-\alpha)\,\theta + \alpha\,\tilde\theta$
      \hfill \textcolor{gray}{// running‐average smoothing to account for noise}

\EndFor
\State \Return $P\equiv\theta$
\end{algorithmic}
\label{alg:sharpness-aware-em}
\end{algorithm}


\section{Experimental Setup}
In this section, we provide the details pertaining to the datasets, model architectures, training procedures, hyperparameters, and implementation that underlie the experiments reported in the main paper.

\subsection{Datasets}
\begin{figure*}[h]
    \centering
    \includegraphics[width = 0.19\linewidth]{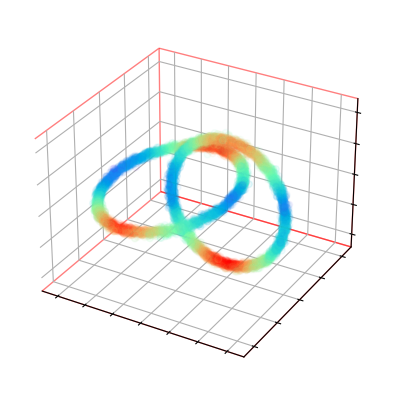}
    \includegraphics[width = 0.19\linewidth]{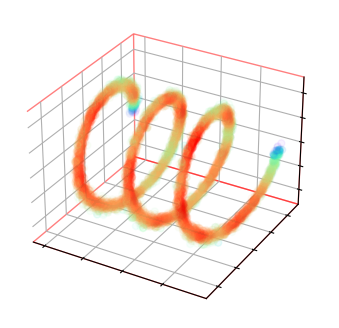}
    \includegraphics[width = 0.19\linewidth]{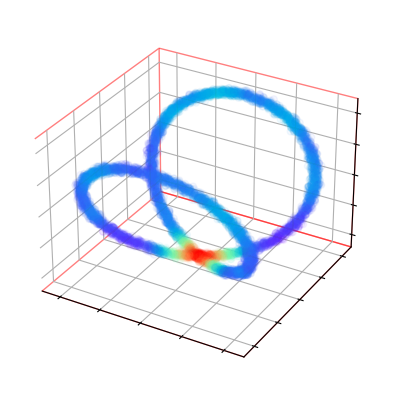}
    \includegraphics[width = 0.19\linewidth]{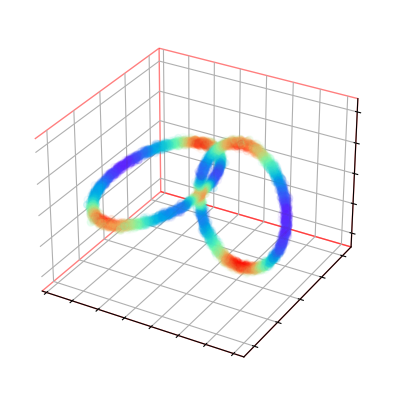}
    \includegraphics[width = 0.19\linewidth]{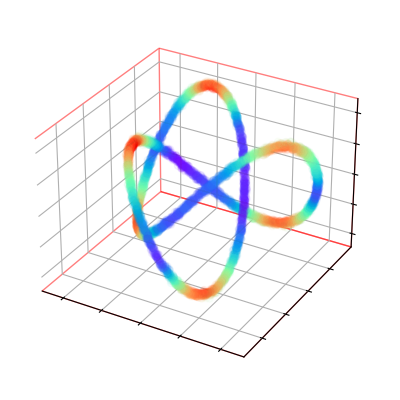}
    \begin{tikzpicture}
         \node[rectangle,draw=white!90,fill=gray!10,opacity=0.9,minimum width=3.5cm,minimum height=0.5cm] at (-0.5,0) {\small{Interlocked-Circles}};
         \node[rectangle,draw=white!90,fill=gray!10,opacity=0.9,minimum width=3.0cm,minimum height=0.5cm] at (3.0,0) {\small{Helix}};
         \node[rectangle,draw=white!90,fill=gray!10,opacity=0.9,minimum width=3cm,minimum height=0.5cm] at (6.3,0) {\small{Bent-Lissajous}};
         \node[rectangle,draw=white!90,fill=gray!10,opacity=0.9,minimum width=3cm,minimum height=0.5cm] at (9.8,0) {\small{Twisted-Eight}};
         \node[rectangle,draw=white!90,fill=gray!10,opacity=0.9,minimum width=3cm,minimum height=0.5cm] at (13.2,0) {\small{Knotted}};
     \end{tikzpicture}

     \includegraphics[width = 0.19\linewidth]{ 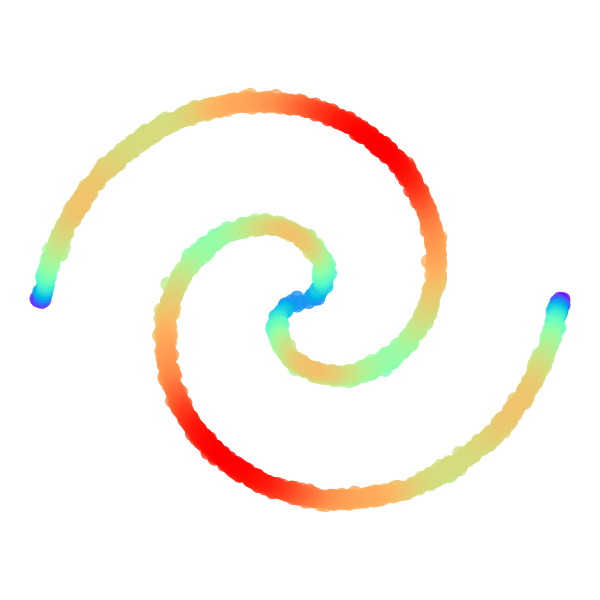}
    \includegraphics[width = 0.19\linewidth]{ 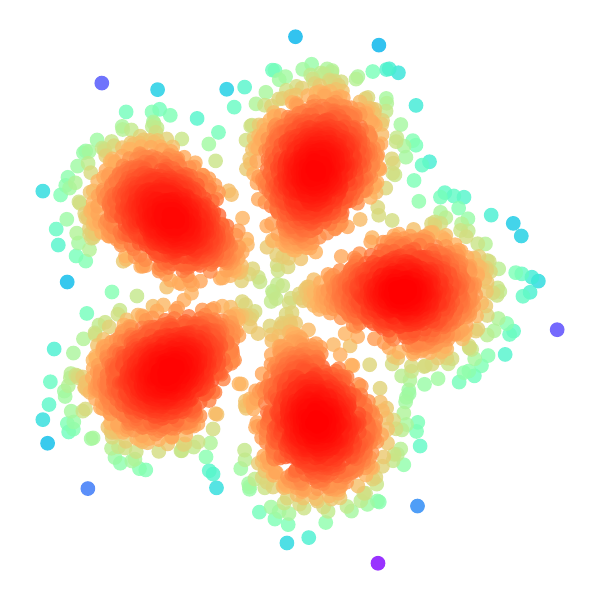}
    \includegraphics[width = 0.19\linewidth]{ 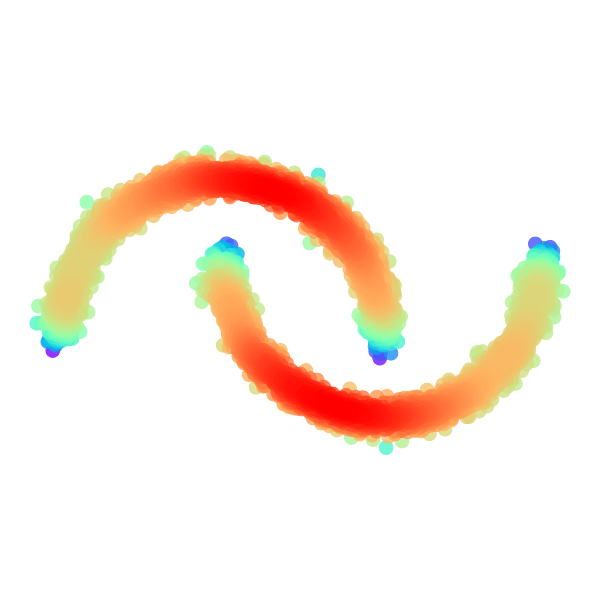}
    
    \begin{tikzpicture}
         \node[rectangle,draw=white!90,fill=white,opacity=1.0,minimum width=3.5cm,minimum height=0.5cm] at (-0.5,0) {};
         \node[rectangle,draw=white!90,fill=gray!10,opacity=0.9,minimum width=3.0cm,minimum height=0.5cm] at (3.0,0) {\small{Spiral}};
         \node[rectangle,draw=white!90,fill=gray!10,opacity=0.9,minimum width=3cm,minimum height=0.5cm] at (6.3,0) {\small{Pinwheel}};
         \node[rectangle,draw=white!90,fill=gray!10,opacity=0.9,minimum width=3cm,minimum height=0.5cm] at (9.8,0) {\small{Two Moons}};
         \node[rectangle,draw=white!90,fill=white,opacity=1.0,minimum width=3cm,minimum height=0.5cm] at (13.2,0) {};
     \end{tikzpicture}
    \caption{Visualizations of the $8$ synthetic \textbf{$3$D} (top) and \textbf{$2$D} (bottom) manifold data distributions used in our empirical analysis.}
    \label{synthetic-datasets}
\end{figure*}

\paragraph{Low Dimensional Data.} We first consider $8$ synthetic data distributions over $3D$ and $2D$ manifolds. These datasets have been shown to be challenging to learn for several generative models \cite{sidheekh2022vq,SidheekhPMLR2023}, and their low dimensionality allows us to verify the correctness of our derivations by computing the full hessian trace via Pytorch autograd. Figure \ref{synthetic-datasets} provides a visualization of these datasets. For each manifold, we sample $1000$ datapoints each for training, validation and testing.

\begin{table}[h!]
\centering
\caption{Overview of the $20$ binary density estimation datasets, showing the number of variables and the number of instances in the training, validation, and test splits.}
\label{tab:20DEBD}
\begin{tabular}{lrrrr}
\toprule
\textbf{Dataset Name}       & \textbf{\#vars}     & \textbf{\#train}       & \textbf{\#valid }      & \textbf{\#test}        \\
\midrule
nltcs         & $16$       & $16181$       & $2157$        & $3236$        \\
msnbc         & $17$       & $291326$      & $38843$       & $58265$       \\
kdd           & $65$       & $180092$      & $19907$       & $34955$       \\
plants        & $69$       & $17412$       & $2321$        & $3482$        \\
baudio        & $100$      & $15000$       & $2000$        & $3000$        \\
jester        & $100$      & $9000$        & $1000$        & $4116$        \\
bnetflix      & $100$      & $15000$       & $2000$        & $3000$        \\
accidents     & $111$      & $12758$       & $1700$        & $2551$        \\
tretail       & $135$      & $22041$       & $2938$        & $4408$        \\
pumsb\_star   & $163$      & $12262$       & $1635$        & $2452$        \\
dna           & $180$      & $1600$        & $400$         & $1186$        \\
kosarek       & $190$      & $33375$       & $4450$        & $6675$        \\
msweb         & $294$      & $29441$       & $3270$        & $5000$        \\
tmovie        & $500$      & $4524$        & $1002$        & $591$         \\
book          & $500$      & $8700$        & $1159$        & $1739$        \\
cwebkb        & $839$      & $2803$        & $558$         & $838$         \\
cr52          & $889$      & $6532$        & $1028$        & $1540$        \\
c20ng         & $910$      & $11293$       & $3764$        & $3764$        \\
bbc           & $1058$     & $1670$        & $225$         & $330$         \\
ad            & $1556$     & $2461$        & $327$         & $491$         \\
\bottomrule
\end{tabular}
\end{table}

\paragraph{Real World Data.} To study the applicability of our framework in more complex, higher dimensional real world domains, we consider the standard suite of $20$ binary density estimation benchmark \cite{van2012markov, bekker2015tractable}. These include small to large domains such as \texttt{nltcs} ($16$ variables), \texttt{msnbc} ($17$ variables), up to \texttt{ad} ($1556$ variables). Table \ref{tab:20DEBD} summarizes the number of variables as well as the number of datapoints present in the train, validation and test split for each of these $20$ datasets.

\subsection{Model Architectures}
To demonstrate the applicability of our framework across different PC structures and implementations, we embed our sharpness-aware regularizer into two popular Probabilistic Circuit frameworks for each dataset type:
\begin{enumerate}
    \item \noindent \textbf{Einsum Networks} \cite{PeharzICML2020} for our experiments on synthetic data, which uses a Random-Tensorized structure introduced in RAT-SPNs \cite{PeharzUAI2019} for the PC, and combines the sum and product operations into a tensorized \textit{einsum} operation that is more efficient and scalable on GPUs.  We employ the “binary” graph type (variable bipartitions) to generate the random structure.  The hyperparameters defining the model structure and capacity include - 
    \begin{itemize}
        \item \texttt{num\_vars}: The number of observed random variables. We set this to $2$ (or $3$) for our synthetic $2D$ (or $3D$) manifolds.
        \item \texttt{num\_input\_distributions}: The number of input leaf distributions per variable. We set this to $10$, yielding $10$ \textit{Gaussians} per input dimension for our synthetic $2D$ and $3D$ datasets.
        \item \texttt{num\_sums}: The number of sum nodes per layer. We set this to $10$.
        \item \texttt{num\_repetitions}: The number of replicas, which we also set to $10$.
        \item \texttt{depth}: The total number of layers (sum and product alterations) we set this to $1$ for our synthetic experiments as it is the max depth possible for $2D/3D$
 data.
    \end{itemize}
    \item \noindent \textbf{PyJuice} \cite{LiuICML2024}, for our experiments on the binary density estimation datasets. We use the Hidden Chow-Liu Tree structure \citep{LiuNeurIPS2021,LiuICML2024} 
    which is a generative probabilistic model that extends the classical Chow-Liu tree by introducing latent (hidden) variables to model complex dependencies among observed variables more effectively. The tree topology is learned from data using maximum‐likelihood (Chow-Liu algorithm). The observed variables are then pushed to the leaves, introducing latent variables to occupy the internal nodes, forming a latent tree structure. As a result, the learned structure can vary across datasets, adapting to the underlying statistical relationships. The latent size (\texttt{num\_latent}) refers to the number of states each hidden variable can take, and it serves as a key hyperparameter that controls the model’s capacity.
    We set \texttt{num\_latent}$=100$ for all our binary density estimation datasets.
\end{enumerate}

\subsection{Training and Implementation Details}
\noindent To simulate low data settings where PCs can overfit and study the effect the regularizer can have, we train the models on multiple random subsets of the official training set for all the datasets, at fractions \{$1$\%, $5$\%, $10$\%, $50$\%, $100$\%\}—always evaluating on the same held‐out validation and test partitions. We used the same model architecture and experimental setup for the base model and the regularized version for a fair comparison, and repeated each experiment across $5$ independent trials, setting the random seed to be equal to the trial number.

\paragraph{Einsum Networks (Gradient‐Based Learning).}  
For the eight synthetic low‐dimensional datasets, we employed the RAT‐SPN architecture implemented in the Einsum Networks library.  We used \texttt{num\_input\_distributions}=$10$, \texttt{num\_sum\_nodes}=$10$, \texttt{num\_repetitions}=$10$, and \texttt{depth}=$1$ (one sum and one product layer per repetition).  We optimized the parameters using an Adam optimizer with a learning rate=$10^{-1}$, batch size=$200$, for $200$ epochs, using the objective
\[
\mathcal{L}_{reg}(\theta)
\;=\; -\sum_{x\in D_{\mathrm{train}}}\log P_{\mathsf{PC}}(x)
\;+\;\mu
\sum_{n,c}\Bigl(\dfrac{F_{nc}(x)}{\theta_{nc}}\Bigr)^2,
\]
where the second term is our Hessian‐trace regularizer and setting $\mu=0$ recovers the unregularized objective.  We set 
$\lambda=1$ and selected \(\mu\) based on the validation performance (negative log-likelihood) over the grid \(\{0.01,\,0.05\,0.1,\,0.5,\,1.0\}\). In addition to the fixed grid, we also tried an \emph{adaptive} schedule by setting \(\mu_{\mathrm{adaptive}}
=
\kappa^{\mathrm{DoF}}
\;\cdot\;
\alpha\;
\frac{g_{\mathrm{data}}}{g_{\mathrm{reg}}},
\) at the end of every epoch,
where 
$\mathrm{DoF}
\;=\;
100*\left(\tfrac{\bigl|\mathrm{NLL}_{\mathrm{val}}-\mathrm{NLL}_{\mathrm{train}}\bigr|}{\bigl|\mathrm{NLL}_{\mathrm{train}}\bigr|}\right)$ denotes the degree of overfitting and
$
\quad
g_{\mathrm{data}}=\|\nabla_\theta \mathcal{L}\|_2,
\quad
g_{\mathrm{reg}}=\|\nabla_\theta R\|_2,
$ denotes the norm of the gradients of the log-likelihood objective and the regularizer respectively. We used \(\kappa=1.05\) and \(\alpha=1.0\) to  balance the magnitudes of the data and regularizer gradients during learning and to amplify \(\mu\) when overfitting worsens.

\paragraph{PyJuice HCLTs (EM‐Based Learning).}  
For the 20 standard binary density benchmarks \cite{van2012markov,LowdJMLR2014}, we fit Hidden Chow‐Liu Trees using PyJuice with latent size \texttt{num\_latents}=100.  The structure is learned once from the training set via Chow‐Liu, then converted to an HCLT with each observed variable \(X_i\) connected to a categorical latent \(Z_i\) of size 100.  We ran EM for 100 epochs with smoothing factor \(\alpha_{\mathrm{EM}}=0.1\) and batch size=$200$.  In the regularized M‐step, we used the
 quadratic‐update formula in Theorem \ref{thm:Gradient update}.  We selected \(\mu\) from the same fixed grid \(\{0.01,\,0.05\,0.1,\,0.5,\,1.0\}\) using the validation performance and also experimented with a \emph{layer‐wise mean‐flow} schedule: after each E‐step we computed the mean edge‐flow in each layer \(\ell\),
\(\bar F_\ell = \frac1{|\mathcal E_\ell|}\sum_{e\in\mathcal E_\ell}F_e\),
and set \(\mu_\ell=\bar F_\ell\).  As the magnitude of the circuit flows decrease with depth, this schedule prevents over‐regularization of shallow layers or under‐regularization of deep ones.

\paragraph{Compute.} We used $4$ NVIDIA L4 GPUs each with $24$ GB memory to run our experiments. Across roughly $700$ total runs ($8$ synthetic datasets × $5$ fractions × $5$ trials + $20$ binary datasets × $5$ fractions × $5$ trials) the total GPU time was $\approx 192$ hours. The code will be publicly released upon acceptance to ensure reproducibility.

\section{Additional Results}

\subsection{Empirical Motivation: Sharpness and Overfitting in PCs}

\begin{figure}[h!]
    \centering
    \includegraphics[width=0.47\linewidth,clip]{ 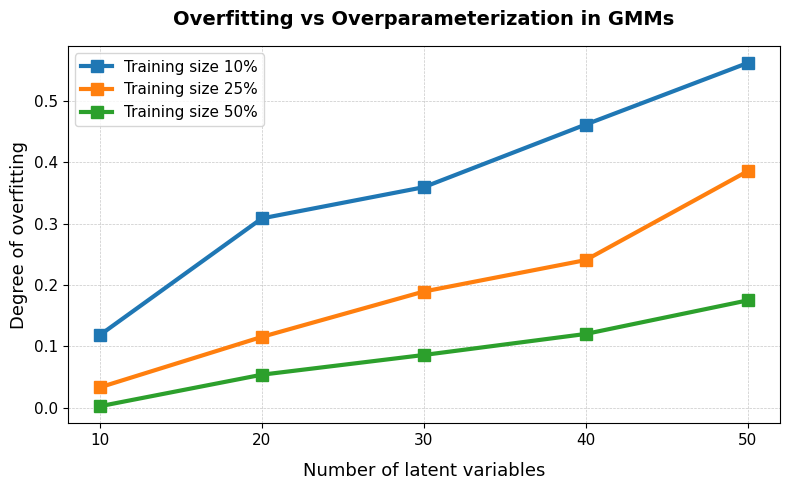}
\includegraphics[width=0.48\linewidth]{ 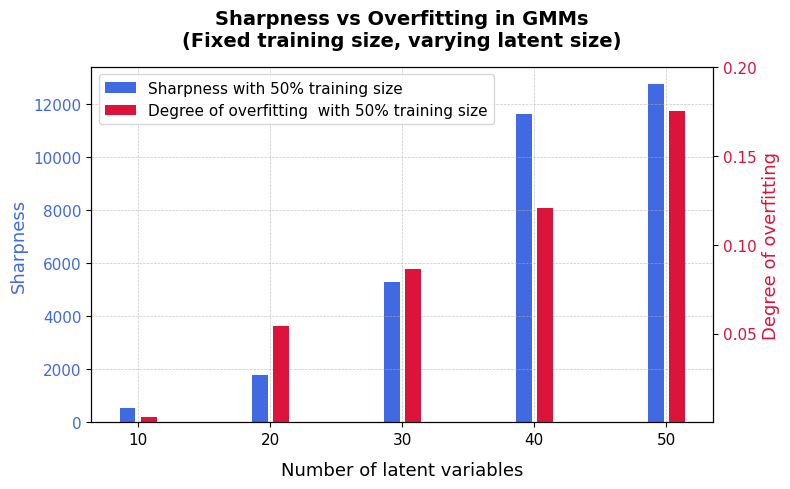}
\caption{
    \textbf{Left:} Degree of overfitting (DoF) vs.\ number of latent variables, for three training fractions ($10$\%, $25$\%, $50$\%). Overfitting grows with model capacity and is worst in low‐data regimes.
    \textbf{Right:} Sharpness (blue, left axis) and Degree of Overfitting (DoF) (red, right axis) as a function of latent size at $50$\% training data. Higher Hessian-trace (sharpness) aligns with larger generalization gaps, motivating our sharpness‐aware regularization framework.
  }
\label{fig:sharpness_overfit}
\end{figure}

\noindent First, we present additional results that concretely motivate our approach, by empirically examining the relationship between sharpness and generalization in PCs. Specifically, we consider a PyJuice HCLT with identical architectures trained on the same dataset (we use a synthetic Gaussian mixture model as the ground truth distribution) and compare them in terms of overfitting, measured by the relative generalization gap (or degree of overfitting): $ = \dfrac{\mathcal{L}_{\text{train}} - \mathcal{L}_{\text{test}}}{|\mathcal{L}_{\text{train}}|}
$, where $\mathcal{L}_{\text{train}}$ and $\mathcal{L}_{\text{test}}$ denote the log-likelihoods on the training and test data respectively. As shown in Figure~\ref{fig:sharpness_overfit}, we observe that models converging to sharper optima—quantified by the trace of the Hessian of the log-likelihood—exhibit greater overfitting. This suggests that minimizing curvature-based sharpness could serve as an effective regularization strategy to reduce overfitting in PCs.

\subsection{Extended Quantitative Results}
\input{ tables/einsum-net-complete-results}
\input{ tables/pyjuice_complete_results}
Tables~\ref{tab:synthetic_test_sgd} and~\ref{tab:synthetic_sharp_sgd} summarize the the final test-set negative log-likelihood and the sharpness of the train negative log-likelihood surface at the converged points achieved by an Einsum Network trained with and without our Hessian‐trace regularizer using gradient descent on the eight synthetic manifold datasets.
Similarly, Tables~\ref{tab:binary_test_em} and~\ref{tab:binary_sharp_em} report the results for PyJuice HCLTs trained on the 20 binary benchmarks using EM with and without our regularizer. In both settings, we observe that adding the regularizer guides convergence to flatter minima and helps achieve better generalization, especially in the lower data regimes.

\subsection{Visualization of Loss Landscapes}

To illustrate how our Hessian-trace regularizer reshapes the optimization geometry, we train EinsumNet on two synthetic benchmarks (\texttt{pinwheel} and \texttt{two\_moons}) using gradient descent with and without our regularization, and plot the loss landscape at the converged points using the visualization algorithm proposed by \cite{li2018visualizing} in Figures \ref{fig:loss-landscape-pinwheel}-\ref{fig:loss-landscape-two_moons}. Specifically,  for each converged solution $\theta^*$, we compute and visualize the following four diagnostics:

\begin{itemize}
  \item \textbf{1D Loss Surface:} $L(\theta^* + \alpha u)$ plotted against the scalar $\alpha$, where $u$ is a random unit vector in parameter space.
  \item \textbf{2D Loss Surface:} $L(\theta^* + \alpha u + \beta v)$ over a grid of $(\alpha,\beta)$, with $u,v$ two orthonormal directions.
  \item \textbf{2D Contour Plot:} Iso‐contours of the 2D loss surface to highlight valley geometry.
  \item \textbf{Hessian Spectrum:} Histogram of the top-$10$ eigenvalues of the Hessian, quantifying the local curvature.
\end{itemize}

\noindent In each figure, the top row corresponds to the unregularized model and the bottom row corresponds to the model trained with our  regularizer.  From left to right, the columns correspond to the four diagnostics above. We can observe a visibly flatter $1$D/$2$D surface and a smaller Hessian spectrum under our regularization framework, which indicates that it guides convergence to a flatter, more stable minima.

\begin{figure}[h]
    \centering
    \begin{tikzpicture}
         \node[rectangle,draw=white!90,fill=gray!10,opacity=0.9,minimum width=3.5cm,minimum height=0.5cm] at (-1.0,0) {\small{$1$D Loss Surface}};
         \node[rectangle,draw=white!90,fill=gray!10,opacity=0.9,minimum width=3.5cm,minimum height=0.5cm] at (3.2,0) {\small{$2$D Loss Surface}};
         \node[rectangle,draw=white!90,fill=gray!10,opacity=0.9,minimum width=3.5cm,minimum height=0.5cm] at (7.5,0) {\small{$2$D Loss Surface Contour}};
         \node[rectangle,draw=white!90,fill=gray!10,opacity=0.9,minimum width=3.5cm,minimum height=0.5cm] at (12,0) {\small{Hessian Eigen-Values}};
     \end{tikzpicture}\\
    \includegraphics[width=0.25\linewidth,valign=t]{ 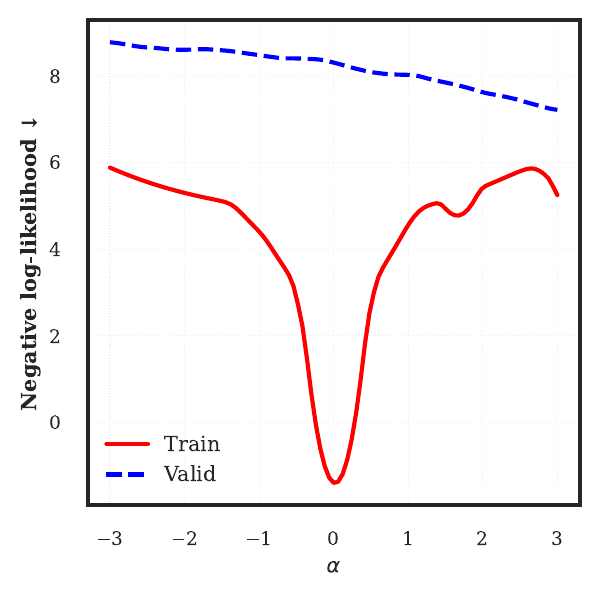}
    \includegraphics[width=0.25\linewidth,valign=t,trim=0em 0em 0em 2em,clip]{ 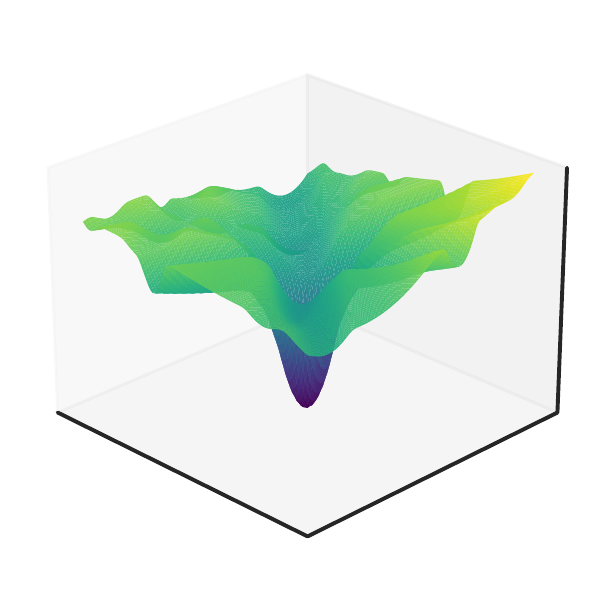}
    \includegraphics[width=0.22\linewidth,valign=t]{ 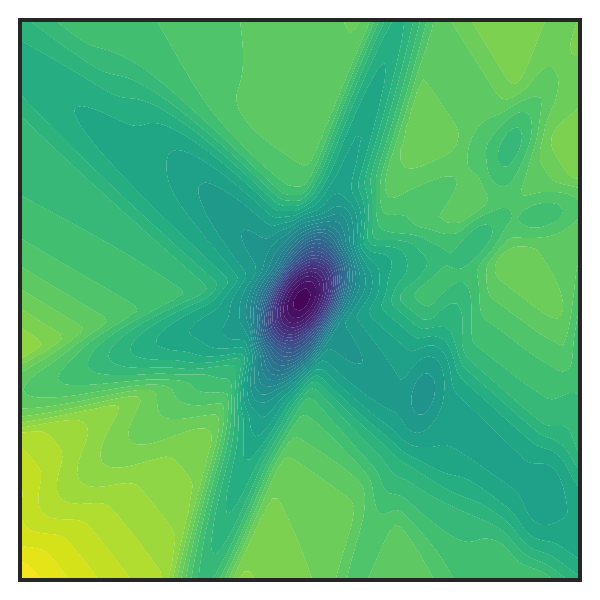}
    \includegraphics[width=0.26\linewidth,valign=t]{ 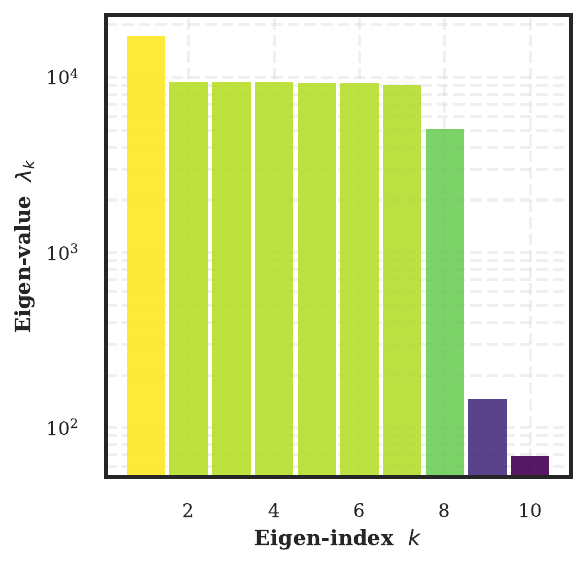}\\
    \includegraphics[width=0.26\linewidth,valign=t]{ 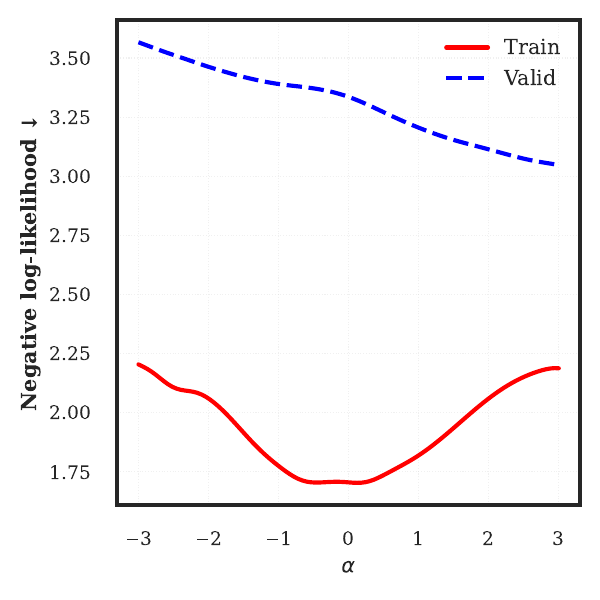}
    \includegraphics[width=0.24\linewidth,valign=t,trim=0em 0em 0em 2em,clip]{ 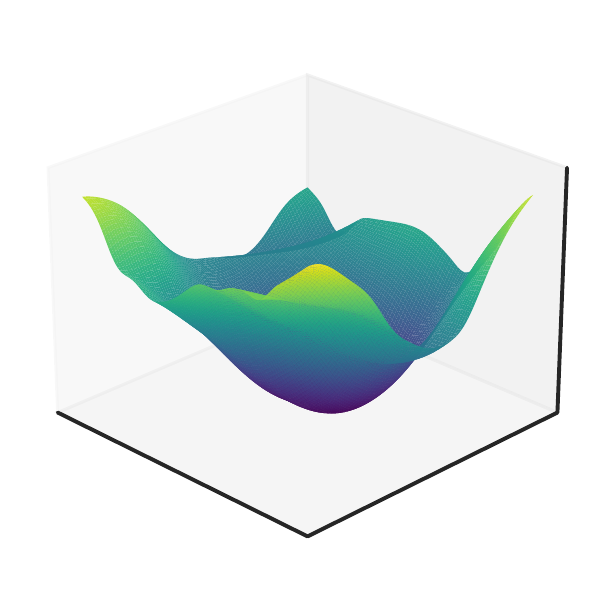}
    \includegraphics[width=0.22\linewidth,valign=t]{ 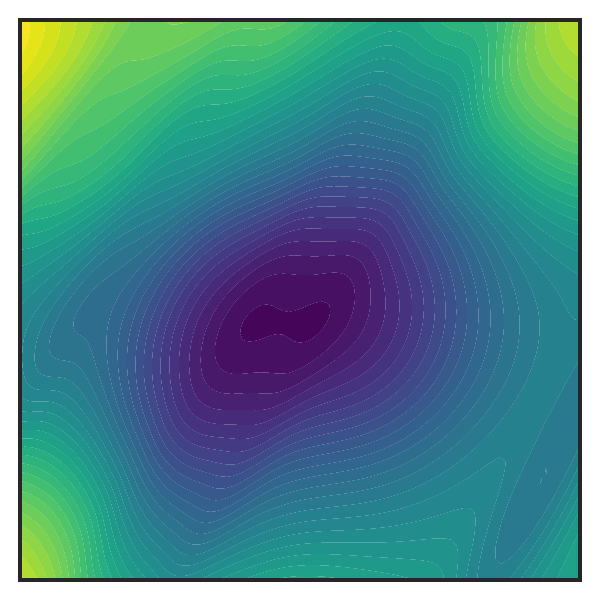}
    \includegraphics[width=0.26\linewidth,valign=t]{ 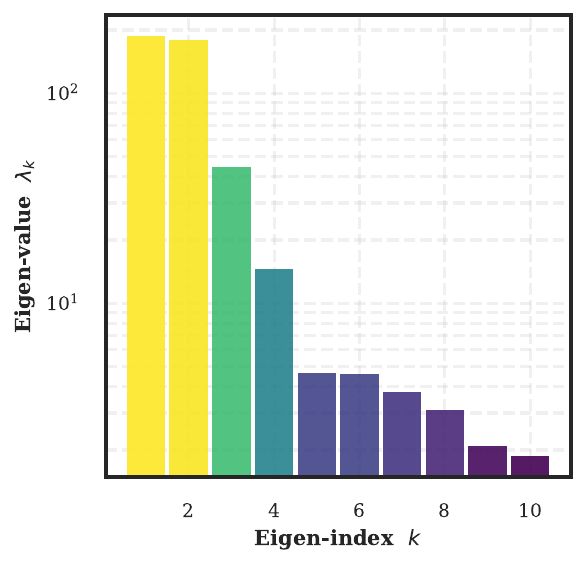}
    \caption{\textbf{Visualization of the loss landscape} geometry and Hessian spectrum for EinsumNet trained on the \textbf{\texttt{pinwheel}} dataset using gradient descent. The top row shows the unregularized model and the bottom row the model trained with our Hessian-trace regularizer. From left to right: (1) one-dimensional loss surface along a random direction in parameter space, $\alpha$ denotes the distance from the converged point. (2) two-dimensional loss surface over a grid of perturbations; (3) contour plot of the 2D loss surface; (4) histogram of Hessian eigenvalues computed at the converged solution. Regularization results in convergence to a flatter valley on the loss surface, validated by the Hessian eigen-spectrum.}
    \label{fig:loss-landscape-pinwheel}
\end{figure}

\begin{figure}[h]
    \centering
    \begin{tikzpicture}
         \node[rectangle,draw=white!90,fill=gray!10,opacity=0.9,minimum width=3.5cm,minimum height=0.5cm] at (-1.0,0) {\small{$1$D Loss Surface}};
         \node[rectangle,draw=white!90,fill=gray!10,opacity=0.9,minimum width=3.5cm,minimum height=0.5cm] at (3.2,0) {\small{$2$D Loss Surface}};
         \node[rectangle,draw=white!90,fill=gray!10,opacity=0.9,minimum width=3.5cm,minimum height=0.5cm] at (7.5,0) {\small{$2$D Loss Surface Contour}};
         \node[rectangle,draw=white!90,fill=gray!10,opacity=0.9,minimum width=3.5cm,minimum height=0.5cm] at (12,0) {\small{Hessian Eigen-Values}};
     \end{tikzpicture}\\
    \includegraphics[width=0.25\linewidth,valign=t]{ figures/loss-landscape/supplementary/two\_moons_EinsumNet_sgd_R0__1d_surface.pdf}
    \includegraphics[width=0.25\linewidth,valign=t,trim=0em 0em 0em 2em,clip]{ figures/loss-landscape/supplementary/two\_moons_EinsumNet_sgd_R0__2d_train_surface.pdf}
    \includegraphics[width=0.22\linewidth,valign=t]{ figures/loss-landscape/supplementary/two\_moons_EinsumNet_sgd_R0__2d_train_surface_contour.pdf}
    \includegraphics[width=0.26\linewidth,valign=t]{ figures/loss-landscape/supplementary/two\_moons_EinsumNet_sgd_R0__hessian_eigs.pdf}\\
    \includegraphics[width=0.26\linewidth,valign=t]{ figures/loss-landscape/supplementary/two\_moons_EinsumNet_sgd_Rauto__1d_surface.pdf}
    \includegraphics[width=0.24\linewidth,valign=t,trim=0em 0em 0em 2em,clip]{ figures/loss-landscape/supplementary/two\_moons_EinsumNet_sgd_Rauto__2d_train_surface.pdf}
    \includegraphics[width=0.22\linewidth,valign=t]{ figures/loss-landscape/supplementary/two\_moons_EinsumNet_sgd_Rauto__2d_train_surface_contour.pdf}
    \includegraphics[width=0.26\linewidth,valign=t]{ figures/loss-landscape/supplementary/two\_moons_EinsumNet_sgd_Rauto__hessian_eigs.pdf}
    \caption{\textbf{Visualization of the loss landscape} geometry and Hessian spectrum for EinsumNet trained on the \textbf{\texttt{two\_moons}} dataset using gradient descent. The top row shows the unregularized model and the bottom row the model trained with our Hessian-trace regularizer. From left to right: (1) one-dimensional loss surface along a random direction in parameter space, $\alpha$ denotes the distance from the converged point. (2) two-dimensional loss surface over a grid of perturbations; (3) contour plot of the 2D loss surface; (4) histogram of Hessian eigenvalues computed at the converged solution. Regularization results in convergence to a flatter valley on the loss surface, validated by the Hessian eigen-spectrum.}
    \label{fig:loss-landscape-two_moons}
\end{figure}

\subsection{Learning Curves and Sharpness}

To gain deeper insights into how the Hessian‐trace regularization affects the training process, we plot the training and validation negative log‐likelihood (NLL) together with the sharpness over epochs for both gradient‐based and EM‐based learning. Each figure displays results for a single dataset under varying fractions of the training data, comparing the unregularized baseline (left column) against our Hessian‐trace regularized model (right column).
Figure~\ref{fig:spiral-sgd} illustrates the case of the 2D \texttt{spiral} manifold trained with gradient descent on an Einsum Network. Under the lower data regimes, the baseline overfits and the sharpness increases. In contrast, the regularized model attains a smoother decrease in both train and validation NLL, with sharpness remaining lower than the base model throughout training. As data size increases, the gap narrows, but regularization still yields flatter optima and modest generalization gains.
Similar trends hold across other synthetic manifolds (e.g., \texttt{pinwheel}, \texttt{two\_moons}, \texttt{helix}), as shown in Figures~\ref{fig:pinwheel-sgd}–\ref{fig:helix-sgd}. 
We also observe analogous dynamics for EM‐based learning of PyJuice HCLTs on the real‐world binary benchmarks. We provide a visualization of the same for a few of the binary datasets in Figures~\ref{fig:baudio-em}-~\ref{fig:accidents-em}.

\pcfig{spiral}{sgd}{EinsumNet}
\pcfig{pinwheel}{sgd}{EinsumNet}
\pcfig{twomoons}{sgd}{EinsumNet}
\pcfig{helix}{sgd}{EinsumNet}
\pcfig{bentlissajous}{sgd}{EinsumNet}
\pcfig{interlockedcircles}{sgd}{EinsumNet}
\pcfig{twistedeight}{sgd}{EinsumNet}
\pcfig{knotted}{sgd}{EinsumNet}
\pcfig{baudio}{em}{HCLT}
\pcfig{plants}{em}{HCLT}
\pcfig{bnetflix}{em}{HCLT}
\pcfig{accidents}{em}{HCLT}


\bibliography{standardized-references}

%% file: tables/einsum-final.tex
\begin{table*}[!t]
\centering
\caption{Percentage improvements in test negative log-likelihood ($\Delta NLL$), reduction in overfitting ($\Delta DoF$), and percentage decrease in the loss-surface sharpness ($\Delta Sharp$) for an \textbf{Einsum Network} on the \textbf{synthetic manifold datasets}, comparing models with Hessian-trace regularization against vanilla counterparts. Values are averaged across $5$ independant runs.}
\label{tab:synthetic_sgd_combined}
\footnotesize%
\setlength{\tabcolsep}{2pt}%
\renewcommand{\arraystretch}{0.9}%
\begin{tabular}{@{}l@{}c@{}c@{}c@{}c@{}c@{}c@{}c@{}c@{}c@{}c@{}c@{}c@{}c@{}c@{}c@{}}
\toprule
\textbf{Dataset} & \multicolumn{3}{c}{\textbf{$\mathbf{1}\%$}} & \multicolumn{3}{c}{\textbf{$\mathbf{5}\%$}} & \multicolumn{3}{c}{\textbf{$\mathbf{10}\%$}} & \multicolumn{3}{c}{\textbf{$\mathbf{50}\%$}} & \multicolumn{3}{c}{\textbf{$\mathbf{100}\%$}} \\
\cmidrule(lr){2-4}\cmidrule(lr){5-7}\cmidrule(lr){8-10}\cmidrule(lr){11-13}\cmidrule(lr){14-16}
 & \hspace{1.5pt}$\Delta$NLL\hspace{1.5pt} & \hspace{1.5pt}$\Delta$DoF\hspace{1.5pt} & \hspace{1.5pt}$\Delta$Sharp\hspace{1.5pt} & \hspace{1.5pt}$\Delta$NLL\hspace{1.5pt} & \hspace{1.5pt}$\Delta$DoF\hspace{1.5pt} & \hspace{1.5pt}$\Delta$Sharp\hspace{1.5pt} & \hspace{1.5pt}$\Delta$NLL\hspace{1.5pt} & \hspace{1.5pt}$\Delta$DoF\hspace{1.5pt} & \hspace{1.5pt}$\Delta$Sharp\hspace{1.5pt} & \hspace{1.5pt}$\Delta$NLL\hspace{1.5pt} & \hspace{1.5pt}$\Delta$DoF\hspace{1.5pt} & \hspace{1.5pt}$\Delta$Sharp\hspace{1.5pt} & \hspace{1.5pt}$\Delta$NLL\hspace{1.5pt} & \hspace{1.5pt}$\Delta$DoF\hspace{1.5pt} & \hspace{1.5pt}$\Delta$Sharp\hspace{1.5pt} \\
\midrule
\textit{bent\_lissajous} & $46.65$ & $65.62$ & $93.41$ & $19.45$ & $28.23$ & $52.87$ & $18.16$ & $22.21$ & $56.44$ & $1.27$ & $2.01$ & $34.58$ & $0.65$ & $-0.06$ & $31.25$\\
\textit{helix} & $32.38$ & $62.87$ & $91.80$ & $17.92$ & $35.29$ & $57.51$ & $10.85$ & $5.48$ & $21.48$ & $8.93$ & $-1.90$ & $13.76$ & $7.33$ & $-1.76$ & $15.71$\\
\textit{interlocked\_circles} & $41.19$ & $68.16$ & $88.06$ & $26.54$ & $39.14$ & $75.72$ & $18.65$ & $14.55$ & $46.30$ & $2.20$ & $1.14$ & $34.15$ & $1.02$ & $0.40$ & $41.59$\\
\textit{knotted} & $52.66$ & $59.88$ & $91.01$ & $26.68$ & $42.31$ & $65.46$ & $17.06$ & $9.42$ & $31.12$ & $5.67$ & $1.23$ & $25.44$ & $10.07$ & $-0.01$ & $21.57$\\
\textit{pinwheel} & $57.29$ & $61.85$ & $93.85$ & $22.62$ & $23.27$ & $66.14$ & $13.24$ & $17.58$ & $75.54$ & $3.70$ & $6.14$ & $76.26$ & $0.38$ & $0.77$ & $79.01$\\
\textit{spiral} & $62.11$ & $69.82$ & $92.23$ & $34.03$ & $37.47$ & $67.09$ & $22.62$ & $20.49$ & $61.48$ & $19.83$ & $1.38$ & $12.66$ & $18.35$ & $3.78$ & $23.72$\\
\textit{twisted\_eight} & $49.03$ & $70.20$ & $75.09$ & $13.65$ & $22.31$ & $63.21$ & $11.22$ & $12.54$ & $51.52$ & $2.72$ & $2.42$ & $32.11$ & $-1.22$ & $-0.04$ & $40.85$\\
\textit{two\_moons} & $54.92$ & $67.30$ & $88.47$ & $27.99$ & $30.24$ & $56.63$ & $19.72$ & $17.88$ & $34.10$ & $16.22$ & $-0.13$ & $30.68$ & $3.80$ & $0.85$ & $28.31$\\
\midrule
\textbf{\textit{Mean}} & $49.53$ & $65.71$ & $89.24$ & $23.61$ & $32.28$ & $63.08$ & $16.44$ & $15.02$ & $47.25$ & $7.57$ & $1.54$ & $32.45$ & $5.05$ & $0.49$ & $35.25$\\
\bottomrule
\end{tabular}
\end{table*}

%% file: tables/pyjuice-final.tex
\begin{table*}[!t]
\centering
\caption{Percentage improvements in test negative log-likelihood ($\Delta NLL$), reduction in overfitting ($\Delta DoF$), and percentage decrease in the loss-surface sharpness ($\Delta Sharp$) for a \textbf{PyJuice HCLT} model on the \textbf{binary density estimation datasets}, comparing models with Hessian-trace regularization against vanilla counterparts. Values represent the mean over $5$ runs.}
\label{tab:binary_em_combined}
\footnotesize%
\setlength{\tabcolsep}{2pt}%
\renewcommand{\arraystretch}{0.9}%
\begin{tabular}{@{}l@{}c@{}c@{}c@{}c@{}c@{}c@{}c@{}c@{}c@{}c@{}c@{}c@{}c@{}c@{}c@{}}
\toprule
\textbf{Dataset} & \multicolumn{3}{c}{\textbf{$\mathbf{1}\%$}} & \multicolumn{3}{c}{\textbf{$\mathbf{5}\%$}} & \multicolumn{3}{c}{\textbf{$\mathbf{10}\%$}} & \multicolumn{3}{c}{\textbf{$\mathbf{50}\%$}} & \multicolumn{3}{c}{\textbf{$\mathbf{100}\%$}} \\
\cmidrule(lr){2-4}\cmidrule(lr){5-7}\cmidrule(lr){8-10}\cmidrule(lr){11-13}\cmidrule(lr){14-16}
 & \hspace{1.5pt}$\Delta$NLL\hspace{1.5pt} & \hspace{1.5pt}$\Delta$DoF\hspace{1.5pt} & \hspace{1.5pt}$\Delta$Sharp\hspace{1.5pt} & \hspace{1.5pt}$\Delta$NLL\hspace{1.5pt} & \hspace{1.5pt}$\Delta$DoF\hspace{1.5pt} & \hspace{1.5pt}$\Delta$Sharp\hspace{1.5pt} & \hspace{1.5pt}$\Delta$NLL\hspace{1.5pt} & \hspace{1.5pt}$\Delta$DoF\hspace{1.5pt} & \hspace{1.5pt}$\Delta$Sharp\hspace{1.5pt} & \hspace{1.5pt}$\Delta$NLL\hspace{1.5pt} & \hspace{1.5pt}$\Delta$DoF\hspace{1.5pt} & \hspace{1.5pt}$\Delta$Sharp\hspace{1.5pt} & \hspace{1.5pt}$\Delta$NLL\hspace{1.5pt} & \hspace{1.5pt}$\Delta$DoF\hspace{1.5pt} & \hspace{1.5pt}$\Delta$Sharp\hspace{1.5pt} \\
\midrule
\textit{accidents} & $1.40$ & $6.94$ & $15.61$ & $-1.88$ & $1.54$ & $16.54$ & $-1.67$ & $0.67$ & $12.33$ & $-0.50$ & $0.04$ & $12.03$ & $-0.29$ & $0.01$ & $3.39$\\
\textit{ad} & $2.11$ & $1.24$ & $56.79$ & $3.20$ & $3.54$ & $78.44$ & $1.26$ & $2.65$ & $27.21$ & $-2.26$ & $1.42$ & $10.80$ & $-1.85$ & $0.43$ & $9.31$\\
\textit{baudio} & $9.39$ & $24.87$ & $9.33$ & $-0.58$ & $1.30$ & $3.06$ & $-0.56$ & $0.35$ & $5.17$ & $-0.20$ & $0.00$ & $1.19$ & $-0.11$ & $0.00$ & $0.88$\\
\textit{bbc} & $10.45$ & $3.52$ & $32.12$ & $15.75$ & $19.70$ & $25.16$ & $8.75$ & $19.57$ & $28.21$ & $-0.23$ & $1.32$ & $11.47$ & $-0.37$ & $0.25$ & $10.74$\\
\textit{bnetflix} & $3.07$ & $10.71$ & $-1.65$ & $-0.38$ & $0.29$ & $2.94$ & $-0.31$ & $0.12$ & $2.68$ & $-0.02$ & $0.00$ & $0.30$ & $-0.01$ & $0.00$ & $0.14$\\
\textit{book} & $9.56$ & $6.11$ & $74.56$ & $1.47$ & $4.15$ & $29.80$ & $-0.60$ & $1.09$ & $21.63$ & $-0.83$ & $0.03$ & $10.23$ & $-0.48$ & $0.02$ & $7.63$\\
\textit{c20ng} & $11.56$ & $13.28$ & $31.06$ & $0.78$ & $3.03$ & $25.84$ & $0.06$ & $0.84$ & $10.53$ & $-0.28$ & $0.01$ & $6.30$ & $-0.12$ & $0.01$ & $3.54$\\
\textit{cr52} & $10.54$ & $9.76$ & $35.21$ & $3.68$ & $12.54$ & $-2.54$ & $-0.10$ & $4.04$ & $3.47$ & $-0.57$ & $0.24$ & $3.45$ & $-0.31$ & $0.04$ & $2.65$\\
\textit{cwebkb} & $9.05$ & $4.05$ & $30.95$ & $14.17$ & $26.46$ & $30.52$ & $3.83$ & $12.36$ & $28.79$ & $-0.63$ & $0.39$ & $11.21$ & $-0.46$ & $0.11$ & $5.29$\\
\textit{dna} & $29.45$ & $15.54$ & $3.87$ & $2.59$ & $6.88$ & $13.93$ & $0.10$ & $2.30$ & $8.95$ & $-0.89$ & $0.25$ & $8.81$ & $-0.47$ & $0.07$ & $3.91$\\
\textit{jester} & $19.22$ & $28.28$ & $25.52$ & $0.24$ & $3.81$ & $-2.60$ & $-0.71$ & $0.36$ & $-3.30$ & $-0.24$ & $0.04$ & $0.88$ & $-0.12$ & $0.00$ & $-0.89$\\
\textit{kdd} & $0.04$ & $1.51$ & $11.96$ & $0.11$ & $0.16$ & $3.50$ & $0.10$ & $0.08$ & $7.49$ & $-0.02$ & $0.00$ & $3.65$ & $0.00$ & $0.00$ & $2.20$\\
\textit{kosarek} & $1.93$ & $6.84$ & $23.70$ & $-0.76$ & $0.60$ & $13.33$ & $-0.42$ & $0.27$ & $7.90$ & $-0.20$ & $0.03$ & $7.51$ & $-0.07$ & $0.00$ & $5.66$\\
\textit{msnbc} & $-0.10$ & $0.03$ & $0.35$ & $0.02$ & $-0.01$ & $0.04$ & $-0.02$ & $0.00$ & $0.07$ & $-0.01$ & $0.00$ & $0.06$ & $0.00$ & $0.00$ & $0.06$\\
\textit{msweb} & $1.11$ & $3.75$ & $33.03$ & $-0.23$ & $0.79$ & $18.57$ & $-0.29$ & $0.45$ & $7.33$ & $-0.08$ & $0.04$ & $10.80$ & $-0.03$ & $0.03$ & $6.54$\\
\textit{nltcs} & $0.18$ & $2.31$ & $7.20$ & $-0.94$ & $0.35$ & $4.80$ & $-0.68$ & $0.06$ & $2.52$ & $-0.08$ & $0.00$ & $-1.04$ & $-0.07$ & $0.00$ & $-0.58$\\
\textit{plants} & $1.59$ & $7.19$ & $13.72$ & $-2.41$ & $0.95$ & $10.34$ & $-1.63$ & $0.38$ & $8.02$ & $-0.56$ & $0.02$ & $3.56$ & $-0.28$ & $0.01$ & $3.50$\\
\textit{pumsb\_star} & $2.04$ & $4.83$ & $8.62$ & $-2.76$ & $1.09$ & $16.12$ & $-1.69$ & $0.51$ & $8.68$ & $-0.50$ & $0.05$ & $5.50$ & $-0.29$ & $0.02$ & $4.33$\\
\textit{tmovie} & $15.15$ & $15.90$ & $16.97$ & $12.01$ & $28.27$ & $-3.17$ & $0.02$ & $8.32$ & $9.83$ & $-1.42$ & $0.34$ & $3.55$ & $-0.82$ & $0.12$ & $11.08$\\
\textit{tretail} & $4.55$ & $9.60$ & $31.95$ & $-0.25$ & $0.44$ & $21.81$ & $-0.10$ & $0.12$ & $10.93$ & $-0.03$ & $0.00$ & $0.69$ & $-0.02$ & $0.00$ & $1.44$\\
\midrule
\textbf{\textit{Mean}} & $7.12$ & $8.81$ & $23.04$ & $2.19$ & $5.79$ & $15.32$ & $0.27$ & $2.73$ & $10.42$ & $-0.48$ & $0.21$ & $5.55$ & $-0.31$ & $0.06$ & $4.04$\\
\bottomrule
\end{tabular}
\end{table*}

%% file: tables/einsum-net-complete-results.tex
\begin{table}[h!]
\centering
\caption{Final \textbf{negative log-likelihoods} on the \textbf{test} split achieved by \textbf{Einsum Network} trained with and without the Hessian-trace regularizer using gradient descent on the \textbf{synthetic} $2D$ and $3D$ manifold datasets. Values represent mean over $5$ independent trials.}
\label{tab:synthetic_test_sgd}
\setlength{\tabcolsep}{2pt}
\renewcommand{\arraystretch}{0.9}
\begin{tabular}{@{}l@{}c@{}c@{}c@{}c@{}c@{}c@{}c@{}c@{}c@{}c@{}}
\toprule
\textbf{Dataset} & \multicolumn{2}{c}{\textbf{1\%}} & \multicolumn{2}{c}{\textbf{5\%}} & \multicolumn{2}{c}{\textbf{10\%}} & \multicolumn{2}{c}{\textbf{50\%}} & \multicolumn{2}{c}{\textbf{100\%}}\\
\cmidrule(lr){2-3}\cmidrule(lr){4-5}\cmidrule(lr){6-7}\cmidrule(lr){8-9}\cmidrule(lr){10-11}
 & \hspace{2pt} Base\hspace{2pt} & Regularized\hspace{2pt} & Base\hspace{2pt} & Regularized\hspace{2pt} & Base\hspace{2pt} & Regularized\hspace{2pt} & Base\hspace{2pt} & Regularized\hspace{2pt} & Base\hspace{2pt} & Regularized\hspace{2pt}\\
\midrule
\textit{bent\_lissajous} & $11.482$ & $5.973$ & $3.686$ & $2.957$ & $3.277$ & $2.670$ & $2.176$ & $2.147$ & $2.007$ & $1.993$\\
\textit{helix} & $9.996$ & $6.730$ & $3.894$ & $3.192$ & $2.902$ & $2.575$ & $1.991$ & $1.802$ & $1.590$ & $1.475$\\
\textit{interlocked\_circles \ } & $10.652$ & $6.257$ & $5.014$ & $3.682$ & $3.862$ & $3.141$ & $2.563$ & $2.507$ & $2.456$ & $2.431$\\
\textit{knotted} & $10.869$ & $5.145$ & $5.472$ & $4.000$ & $3.822$ & $3.164$ & $2.235$ & $2.103$ & $1.915$ & $1.719$\\
\textit{pinwheel} & $8.914$ & $3.814$ & $3.700$ & $2.854$ & $2.842$ & $2.464$ & $2.355$ & $2.266$ & $2.247$ & $2.239$\\
\textit{spiral} & $9.663$ & $3.682$ & $3.932$ & $2.629$ & $2.725$ & $2.047$ & $1.497$ & $1.218$ & $1.492$ & $1.244$\\
\textit{twisted\_eight} & $10.622$ & $5.425$ & $4.067$ & $3.475$ & $3.098$ & $2.742$ & $2.162$ & $2.102$ & $2.051$ & $2.075$\\
\textit{two\_moons} & $8.639$ & $3.902$ & $3.184$ & $2.260$ & $2.379$ & $1.902$ & $1.730$ & $1.447$ & $1.450$ & $1.393$\\
\bottomrule
\end{tabular}
\end{table}

\begin{table}[h!]
\centering
\caption{\textbf{Sharpness} of the train log-likelihood surface at the converged parameters of an \textbf{Einsum Network} trained with and without the Hessian-trace regularizer using gradient descent on the \textbf{synthetic} $2D$ and $3D$ manifold datasets. Values represent mean over $5$ independent trials.}
\label{tab:synthetic_sharp_sgd}
\setlength{\tabcolsep}{2pt}
\renewcommand{\arraystretch}{0.9}
\begin{tabular}{@{}l@{}c@{}c@{}c@{}c@{}c@{}c@{}c@{}c@{}c@{}c@{}}
\toprule
\textbf{Dataset} & \multicolumn{2}{c}{\textbf{1\%}} & \multicolumn{2}{c}{\textbf{5\%}} & \multicolumn{2}{c}{\textbf{10\%}} & \multicolumn{2}{c}{\textbf{50\%}} & \multicolumn{2}{c}{\textbf{100\%}}\\
\cmidrule(lr){2-3}\cmidrule(lr){4-5}\cmidrule(lr){6-7}\cmidrule(lr){8-9}\cmidrule(lr){10-11}
 & Base\hspace{2pt} & Regularized\hspace{2pt} & Base\hspace{2pt} & Regularized\hspace{2pt} & Base\hspace{2pt} & Regularized\hspace{2pt} & Base\hspace{2pt} & Regularized\hspace{2pt} & Base\hspace{2pt} & Regularized\hspace{2pt}\\
\midrule
\textit{bent\_lissajous} & $1.279$ & $0.084$ & $1.048$ & $0.480$ & $1.319$ & $0.579$ & $1.723$ & $1.124$ & $1.123$ & $0.770$\\
\textit{helix} & $1.145$ & $0.097$ & $1.008$ & $0.423$ & $1.233$ & $0.968$ & $2.104$ & $1.811$ & $1.288$ & $1.082$\\
\textit{interlocked\_circles\ } & $1.142$ & $0.138$ & $1.137$ & $0.276$ & $1.237$ & $0.665$ & $1.826$ & $1.205$ & $1.118$ & $0.652$\\
\textit{knotted} & $1.016$ & $0.084$ & $1.262$ & $0.439$ & $1.359$ & $0.935$ & $2.121$ & $1.580$ & $1.373$ & $1.077$\\
\textit{pinwheel} & $1.112$ & $0.068$ & $1.066$ & $0.362$ & $1.010$ & $0.238$ & $1.527$ & $0.368$ & $0.862$ & $0.180$\\
\textit{spiral} & $1.191$ & $0.129$ & $1.273$ & $0.428$ & $1.285$ & $0.596$ & $2.094$ & $1.754$ & $1.220$ & $0.892$\\
\textit{twisted\_eight} & $1.113$ & $0.289$ & $1.019$ & $0.383$ & $1.168$ & $0.562$ & $1.766$ & $1.199$ & $1.068$ & $0.633$\\
\textit{two\_moons} & $1.048$ & $0.124$ & $0.900$ & $0.393$ & $1.201$ & $0.776$ & $1.725$ & $1.160$ & $1.128$ & $0.813$\\
\bottomrule
\end{tabular}
\end{table}

%% file: tables/pyjuice_complete_results.tex
\begin{table}[h!]
\centering
\caption{Final \textbf{negative log-likelihoods} on the \textbf{test} split achieved by the \textbf{PyJuice HCLT} models trained with and without the Hessian-trace regularizer using EM on the \textbf{binary} density estimation benchmark. Values represent mean over $5$ independent trials.}
\label{tab:binary_test_em}
\setlength{\tabcolsep}{2pt}
\renewcommand{\arraystretch}{1.0}
\begin{tabular}{@{}l@{}c@{}c@{}c@{}c@{}c@{}c@{}c@{}c@{}c@{}c@{}}
\toprule
\textbf{Dataset} & \multicolumn{2}{c}{\textbf{1\%}} & \multicolumn{2}{c}{\textbf{5\%}} & \multicolumn{2}{c}{\textbf{10\%}} & \multicolumn{2}{c}{\textbf{50\%}} & \multicolumn{2}{c}{\textbf{100\%}}\\
\cmidrule(lr){2-3}\cmidrule(lr){4-5}\cmidrule(lr){6-7}\cmidrule(lr){8-9}\cmidrule(lr){10-11}
 & Base\hspace{2pt} & Regularized\hspace{2pt} & Base\hspace{2pt} & Regularized\hspace{2pt} & Base\hspace{2pt} & Regularized\hspace{2pt} & Base\hspace{2pt} & Regularized\hspace{2pt} & Base\hspace{2pt} & Regularized\hspace{2pt}\\
\midrule
\textit{accidents} & $34.462$ & $33.978$ & $31.004$ & $31.585$ & $29.623$ & $30.117$ & $29.935$ & $30.085$ & $30.400$ & $30.489$\\
\textit{ad} & $113.631$ & $111.227$ & $64.832$ & $62.758$ & $47.492$ & $46.896$ & $26.138$ & $26.717$ & $23.182$ & $23.616$\\
\textit{baudio} & $51.133$ & $46.328$ & $42.741$ & $42.988$ & $42.565$ & $42.802$ & $42.199$ & $42.282$ & $42.362$ & $42.407$\\
\textit{bbc} & $865.477$ & $775.021$ & $413.546$ & $348.391$ & $318.145$ & $290.308$ & $262.896$ & $263.503$ & $261.152$ & $262.113$\\
\textit{bnetflix} & $63.970$ & $62.004$ & $59.488$ & $59.716$ & $59.225$ & $59.407$ & $59.219$ & $59.233$ & $59.050$ & $59.059$\\
\textit{book} & $76.562$ & $69.243$ & $43.391$ & $42.753$ & $38.554$ & $38.785$ & $37.876$ & $38.191$ & $37.635$ & $37.815$\\
\textit{c20ng} & $284.407$ & $251.542$ & $181.128$ & $179.716$ & $174.452$ & $174.348$ & $174.086$ & $174.574$ & $162.192$ & $162.389$\\
\textit{cr52} & $207.083$ & $185.249$ & $114.139$ & $109.935$ & $101.660$ & $101.759$ & $95.856$ & $96.398$ & $96.311$ & $96.614$\\
\textit{cwebkb} & $433.524$ & $394.291$ & $231.935$ & $199.075$ & $182.514$ & $175.492$ & $162.642$ & $163.660$ & $161.588$ & $162.333$\\
\textit{dna} & $314.576$ & $221.919$ & $91.664$ & $89.294$ & $84.991$ & $84.906$ & $82.339$ & $83.069$ & $82.280$ & $82.665$\\
\textit{jester} & $79.907$ & $61.741$ & $56.220$ & $56.081$ & $55.459$ & $55.850$ & $54.800$ & $54.870$ & $54.770$ & $54.801$\\
\textit{kdd} & $2.398$ & $2.397$ & $2.301$ & $2.299$ & $2.269$ & $2.267$ & $2.245$ & $2.245$ & $2.194$ & $2.194$\\
\textit{kosarek} & $14.150$ & $13.876$ & $12.254$ & $12.347$ & $12.077$ & $12.127$ & $11.697$ & $11.721$ & $11.649$ & $11.658$\\
\textit{msnbc} & $6.573$ & $6.580$ & $6.563$ & $6.562$ & $6.547$ & $6.549$ & $6.580$ & $6.580$ & $6.551$ & $6.551$\\
\textit{msweb} & $12.147$ & $12.012$ & $11.014$ & $11.039$ & $10.784$ & $10.815$ & $10.652$ & $10.660$ & $10.603$ & $10.606$\\
\textit{nltcs} & $6.588$ & $6.576$ & $6.342$ & $6.401$ & $6.281$ & $6.324$ & $6.316$ & $6.321$ & $6.230$ & $6.235$\\
\textit{plants} & $16.674$ & $16.410$ & $15.256$ & $15.624$ & $15.117$ & $15.362$ & $15.071$ & $15.155$ & $14.980$ & $15.023$\\
\textit{pumsb\_star} & $33.633$ & $32.948$ & $28.041$ & $28.814$ & $27.864$ & $28.335$ & $27.860$ & $27.999$ & $27.762$ & $27.842$\\
\textit{tmovie} & $141.179$ & $119.796$ & $78.834$ & $69.353$ & $62.023$ & $62.012$ & $56.704$ & $57.511$ & $56.263$ & $56.723$\\
\textit{tretail} & $14.133$ & $13.489$ & $11.291$ & $11.318$ & $11.210$ & $11.220$ & $11.140$ & $11.144$ & $11.167$ & $11.169$\\
\bottomrule
\end{tabular}
\end{table}

\begin{table}[h!]
\centering
\caption{\textbf{Sharpness} of the train log-likelihood surface at the converged parameters of \textbf{PyJuice HCLT} models trained with and without the Hessian-trace regularizer using EM on the \textbf{binary} density estimation benchmark. Values represent mean over $5$ independent trials.}
\label{tab:binary_sharp_em}
\setlength{\tabcolsep}{2pt}
\renewcommand{\arraystretch}{1.0}
\begin{tabular}{@{}l@{}c@{}c@{}c@{}c@{}c@{}c@{}c@{}c@{}c@{}c@{}}
\toprule
\textbf{Dataset} & \multicolumn{2}{c}{\textbf{1\%}} & \multicolumn{2}{c}{\textbf{5\%}} & \multicolumn{2}{c}{\textbf{10\%}} & \multicolumn{2}{c}{\textbf{50\%}} & \multicolumn{2}{c}{\textbf{100\%}}\\
\cmidrule(lr){2-3}\cmidrule(lr){4-5}\cmidrule(lr){6-7}\cmidrule(lr){8-9}\cmidrule(lr){10-11}
 & Base\hspace{2pt} & Regularized\hspace{2pt} & Base\hspace{2pt} & Regularized\hspace{2pt} & Base\hspace{2pt} & Regularized\hspace{2pt} & Base\hspace{2pt} & Regularized\hspace{2pt} & Base\hspace{2pt} & Regularized\hspace{2pt}\\
\midrule
\textit{accidents} & $5.418$ & $4.569$ & $5.377$ & $4.481$ & $1.551$ & $1.351$ & $0.170$ & $0.148$ & $0.078$ & $0.075$\\
\textit{ad} & $50.032$ & $21.350$ & $41.110$ & $8.823$ & $77.544$ & $56.467$ & $43.930$ & $39.190$ & $11.437$ & $10.407$\\
\textit{baudio} & $4.468$ & $4.049$ & $1.226$ & $1.189$ & $0.978$ & $0.927$ & $0.186$ & $0.183$ & $0.052$ & $0.051$\\
\textit{bbc} & $87.328$ & $59.199$ & $39.811$ & $29.792$ & $19.994$ & $14.352$ & $21.743$ & $19.306$ & $5.773$ & $5.153$\\
\textit{bnetflix} & $4.011$ & $4.076$ & $1.098$ & $1.066$ & $0.824$ & $0.801$ & $0.173$ & $0.173$ & $0.044$ & $0.044$\\
\textit{book} & $41.330$ & $10.331$ & $20.406$ & $14.350$ & $5.624$ & $4.393$ & $0.689$ & $0.617$ & $0.482$ & $0.445$\\
\textit{c20ng} & $25.661$ & $17.686$ & $6.732$ & $4.972$ & $4.399$ & $3.930$ & $2.473$ & $2.318$ & $0.664$ & $0.641$\\
\textit{cr52} & $29.593$ & $19.021$ & $12.734$ & $13.124$ & $18.422$ & $17.791$ & $3.298$ & $3.191$ & $0.918$ & $0.894$\\
\textit{cwebkb} & $33.709$ & $23.239$ & $18.840$ & $13.088$ & $21.762$ & $15.457$ & $403.455$ & $358.781$ & $79.846$ & $75.722$\\
\textit{dna} & $45.327$ & $43.570$ & $12.163$ & $10.469$ & $5.780$ & $5.263$ & $1.173$ & $1.069$ & $0.575$ & $0.552$\\
\textit{jester} & $6.775$ & $8.878$ & $5.182$ & $5.290$ & $1.391$ & $1.439$ & $0.351$ & $0.344$ & $0.087$ & $0.087$\\
\textit{kdd} & $0.439$ & $0.383$ & $3.401$ & $3.281$ & $1.170$ & $1.069$ & $0.038$ & $0.036$ & $0.008$ & $0.008$\\
\textit{kosarek} & $3.619$ & $2.714$ & $1.590$ & $1.377$ & $0.412$ & $0.379$ & $0.147$ & $0.135$ & $0.046$ & $0.043$\\
\textit{msnbc} & $0.267$ & $0.266$ & $0.037$ & $0.037$ & $0.023$ & $0.023$ & $0.010$ & $0.010$ & $0.003$ & $0.003$\\
\textit{msweb} & $5.472$ & $3.633$ & $2.291$ & $1.842$ & $0.389$ & $0.360$ & $0.862$ & $0.657$ & $0.221$ & $0.206$\\
\textit{nltcs} & $2.768$ & $2.569$ & $10.136$ & $9.650$ & $2.813$ & $2.743$ & $0.122$ & $0.124$ & $0.031$ & $0.031$\\
\textit{plants} & $3.410$ & $2.942$ & $1.888$ & $1.693$ & $0.538$ & $0.494$ & $0.144$ & $0.139$ & $0.636$ & $0.614$\\
\textit{pumsb\_star} & $6.899$ & $6.301$ & $18.034$ & $15.130$ & $5.338$ & $4.860$ & $0.237$ & $0.224$ & $0.253$ & $0.241$\\
\textit{tmovie} & $15.935$ & $13.224$ & $32.657$ & $33.657$ & $10.971$ & $9.874$ & $2.751$ & $2.659$ & $1.684$ & $1.416$\\
\textit{tretail} & $22.972$ & $15.087$ & $1.800$ & $1.390$ & $18.382$ & $16.292$ & $0.824$ & $0.821$ & $0.186$ & $0.183$\\
\bottomrule
\end{tabular}
\end{table}